\documentclass[twoside,11pt]{article}
\pdfoutput=1
\usepackage{jmlr2e}
\usepackage{amsmath}

\usepackage{microtype}
\usepackage{enumitem}
\usepackage{multirow}

\DeclareMathOperator*{\argmin}{argmin}

\def \bB {\mathbb{B}}

\def \bE {\mathbb{E}}

\def \bN {\mathbb{N}}

\def \bP {\mathbb{P}}

\def \bR {\mathbb{R}}
\def \bS {\mathbb{S}}

\def \cA {\mathcal{A}}

\def \cC {\mathcal{C}}
\def \cD {\mathcal{D}}
\def \cE {\mathcal{E}}
\def \cF {\mathcal{F}}
\def \cG {\mathcal{G}}
\def \cH {\mathcal{H}}

\def \cL {\mathcal{L}}
\def \cM {\mathcal{M}}
\def \cN {\mathcal{N}}
\def \cO {\mathcal{O}}

\def \cR {\mathcal{R}}

\def \cT {\mathcal{T}}

\def \NN {\mathcal{NN}}

\def \sgn {\,{\rm sgn}\,}
\def \star {\,{\rm star}\,}
\def \RTV {\,{\rm RTV}\,}

\usepackage{lastpage}
\jmlrheading{25}{2024}{1-\pageref{LastPage}}{7/23; Revised
5/24}{5/24}{23-0918}{Yunfei Yang and Ding-Xuan Zhou}
\ShortHeadings{Minimax Rates for Over-parameterized Shallow Networks}{Yang and Zhou}

\begin{document}
\title{Nonparametric Regression Using Over-parameterized Shallow ReLU Neural Networks}

\author{\name Yunfei Yang \thanks{Corresponding author.}
\email yunfyang@cityu.edu.hk \\
\addr Department of Mathematics, City University of Hong Kong
\\ Kowloon, Hong Kong, China
\AND
\name Ding-Xuan Zhou
\email dingxuan.zhou@sydney.edu.au \\
\addr School of Mathematics and Statistics, University of Sydney
\\ Sydney, NSW 2006, Australia
}

\editor{Joseph Salmon}
\maketitle

\begin{abstract}
It is shown that over-parameterized neural networks can achieve minimax optimal rates of convergence (up to logarithmic factors) for learning functions from certain smooth function classes, if the weights are suitably constrained or regularized. Specifically, we consider the nonparametric regression of estimating an unknown $d$-variate function by using shallow ReLU neural networks. It is assumed that the regression function is from the H\"older space with smoothness $\alpha<(d+3)/2$ or a variation space corresponding to shallow neural networks, which can be viewed as an infinitely wide neural network. In this setting, we prove that least squares estimators based on shallow neural networks with certain norm constraints on the weights are minimax optimal, if the network width is sufficiently large. As a byproduct, we derive a new size-independent bound for the local Rademacher complexity of shallow ReLU neural networks, which may be of independent interest.
\end{abstract}


\section{Introduction}

In nonparametric regression, we are given a data set of $n$ samples  $\cD_n=\{(X_i,Y_i)\}_{i=1}^n$, which are independent and identically distributed as a $\bR^d \times \bR$-valued random vector $(X,Y)$. The statistical problem is to estimate the so-called regression function $h(x) = \bE[Y|X=x]$ from the observed data $\cD_n$. One of the most popular estimators is the least squares
\[
\widehat{f}_n \in \argmin_{f\in \cF} \frac{1}{n} \sum_{i=1}^n (f(X_i)- Y_i)^2,
\]
where $\cF$ is a suitably chosen function class. In recent years, due to the breakthrough of deep learning, there is increasing interest in the literature in studying the performance of the least squares when $\cF$ is parameterized by neural networks \citep{bauer2019deep,chen2022nonparametric,jiao2023deep,kohler2021rate,mao2021theory,mao2022approximation,nakada2020adaptive,schmidthieber2020nonparametric,suzuki2019adaptivity,yang2024optimal}.

To analyze rates of convergence for the least squares estimator $\widehat{f}_n$, one often assume that the regression function $h$ is in certain function class $\cH$, which represents prior knowledge on the problem. Typically, we can decompose the error of $\widehat{f}_n$ as 
\begin{equation}\label{error decom}
\bE_X \left[ \left|\widehat{f}_n(X)-h(X) \right|^2 \right] \lesssim \cE_{app}(\cH, \cF) + \cC_n(\cF).
\end{equation}
Here, for two quantities $A$ and $B$, $A \lesssim B$ (or $B \gtrsim A$) denotes the statement that $A\le cB$ for some constant $c>0$ (we will also denote $A \asymp B$ when $A \lesssim B \lesssim A$).
In the error decomposition (\ref{error decom}), the approximation error $\cE_{app}(\cH, \cF)$ quantifies how well we can approximate $h\in \cH$ by functions in $\cF$, and the sample complexity $\cC_n(\cF)$ measures how well the learned function $\widehat{f}_n\in \cF$ generalizes to unseen data. For example, the complexity $\cC_n(\cF)$ is often measured by the covering number, VC dimension \citep{vapnik1971uniform} or Rademacher complexity \citep{bartlett2002rademacher} of $\cF$. In most analysis for neural networks, the approximation error and the complexity are bounded by the size of the networks, e.g. the number of parameters or neurons. In particular, when $\cH = \cH^\alpha$ is the unit ball of H\"older space with smoothness index $\alpha>0$, the results of  \citet{schmidthieber2020nonparametric,kohler2021rate} established the rate $n^{-\frac{2\alpha}{d+2\alpha}}$ (up to logarithmic factors) for deep neural networks, which is known to be minimax optimal \citep{stone1982optimal}. However, in these analyses, one has to use under-parameterized neural networks (the number of parameters is less than the sample size $n$) in order to get optimal trade-offs between $\cE_{app}(\cH, \cF)$ and $\cC_n(\cF)$. But, in practical applications of deep neural networks, the number of parameters is often much larger than the number of samples. Hence, these results cannot fully explain the success of neural networks used in practice. To obtain rates of convergence in the over-parameterized regime, the recent works of \citet{jiao2023approximation} and \citet{yang2024optimal} proposed to bound $\cE_{app}(\cH, \cF)$ and $\cC_n(\cF)$ by the size of the weights, i.e. certain norms of the weight matrices, rather than the size of networks. Using this idea, in \citet{yang2024optimal}, we obtained nearly optimal rates for over-parameterized neural networks.

The purpose of this paper is to show that over-parameterized neural networks can achieve the minimax optimal rates for estimating certain smooth functions, if the weights are suitably constrained or regularized. Specifically, we consider shallow neural networks with ReLU activation
\[
f_\theta(x):= \sum_{i=1}^N a_i \sigma((x^\intercal,1)w_i),\quad x\in\bR^d,
\]
where $\theta := (a_1,w_1^\intercal,\dots,a_N,w_N^\intercal)^\intercal$ denotes the vector of trainable parameters and $\sigma(t):= \max\{t,0\}$ is the rectified linear unit (ReLU). Following the ideas in \citet{jiao2023approximation,yang2024optimal}, we constrain the weights by $\kappa(\theta) := \sum_{i=1}^N |a_i| \|w_i\|_2 \le M$ and study the performance of the constrained least squares
\[
\argmin_{\kappa(\theta) \le M} \frac{1}{n} \sum_{i=1}^n (f_\theta(X_i)- Y_i)^2.
\]
As in \citet{yang2024optimal}, we assume that the regression function $h$ belongs to the unit ball of H\"older class $\cH^\alpha$ with $\alpha<(d+3)/2$ or the unit ball of the variation space $\cF_{\sigma}(1)$, which can be regarded as an infinitely wide neural network with weight constraint $\kappa(\theta) \le 1$ (see Section \ref{sec: variation space} for details). The recent paper \citep{yang2024optimal} proved the rates $n^{-\frac{2\alpha}{d+3+2\alpha}}$ for $\cH^\alpha$ and $n^{-1/2}$ for $\cF_{\sigma}(1)$ up to logarithmic factors, under the condition that $M$ is chosen properly and the network width $N$ is sufficiently large (i.e. weights constrained over-parameterized neural networks). We improve these rates to $n^{-\frac{2\alpha}{d+2\alpha}}$ and $n^{-\frac{d+3}{2d+3}}$, which are minimax optimal for $\cH^\alpha$ and $\cF_{\sigma}(1)$ respectively \citep{stone1982optimal,parhi2023minimax,yang2024optimal}. The improvement is from the use of localization technique \citep{bartlett2005local,koltchinskii2006local} in estimating the complexity $\cC_n(\cF)$ of neural networks. One of our main technical contributions is a new size-independent bound for the local Rademacher complexity of shallow neural networks (see Definition \ref{complexity} and Theorem \ref{local Gc bound}):
\[
\cR_n(\NN(N,M);\delta) \lesssim \frac{\delta^{\frac{3}{d+3}} M^{\frac{d}{d+3}}}{\sqrt{n}} \sqrt{\log(nM/\delta)},
\]
where $\NN(N,M)$ is the set of neural network functions $f_\theta$ with width at most $N$ and constraint $\kappa(\theta)\le M$. Our bound generalizes the bound $\cO(M/\sqrt{n})$ for Rademacher complexity in \citet{golowich2020size} which was used in \citet{yang2024optimal} to control the sample complexity.  

Besides the constrained least squares, we also study the corresponding regularized least squares
\[
\argmin_{\theta} \frac{1}{n} \sum_{i=1}^n (f_\theta(X_i)- Y_i)^2 + \lambda \kappa(\theta),
\]
where $\lambda>0$ is a tunable parameter. It is proven that this optimization problem has essentially the same solutions as the problem
\[
\argmin_{\theta} \frac{1}{n} \sum_{i=1}^n (f_\theta(X_i)- Y_i)^2 + \frac{\lambda}{2} \|\theta\|_2^2.
\]
Our second main result shows that the solutions of these two optimization problems also achieve the minimax optimal rates of learning functions in $\cH^\alpha$ with $\alpha<(d+3)/2$ and $\cF_{\sigma}(1)$, if $\lambda$ is chosen properly. This result gives a theoretical guarantee for the use of some regularization methods, such as \citet{neyshabur2015path}.

The rest of the paper is organized as follows. In Section \ref{sec: variation space}, we introduce some function classes of shallow neural networks and their approximation results. Section \ref{sec: main results} presents our main results on the rates of convergence for over-parameterized neural networks. In Section \ref{sec: discussion}, we give some discussions on our results and related works. The omitted proofs of some theorems and lemmas are given in Section \ref{sec: proofs}.

\section{The Variation Space of Shallow Neural Networks}\label{sec: variation space}

Let $\bB^d= \{x\in \bR^d:\|x\|_2\le 1\}$ be the unit ball of $\bR^d$, and $\bS^d=\{x\in\bR^{d+1}:\|x\|_2=1\}$ be the unit sphere of $\bR^{d+1}$. We denote by $\cM(\bS^d)$ the space of signed Radon measures on $\bS^d$ with the total variation norm $\|\mu\|=|\mu|(\bS^d)$ for $\mu \in \cM(\bS^d)$. It is well known that $\cM(\bS^d)$ is the dual space of the space of continuous functions $C(\bS^d)$ \citep[Section 7.3]{folland1999real}. We are interested in the function class $\cF_\sigma$ that contains continuous functions $f_\mu: \bB^d \to \bR$ of the integral form
\begin{equation}\label{integral form}
f_\mu(x) := \int_{\bS^d} \sigma((x^\intercal,1)v) d\mu(v), \quad x\in \bB^d, \mu\in \cM(\bS^d),
\end{equation}
where $\sigma$ is the ReLU function. Note that the integral representation is not unique for $f\in \cF_\sigma$. For example, if $\mu$ is supported on the set
\begin{equation}\label{S_-}
S_-:=\left\{(x_1,\dots,x_{d+1})^\intercal\in \bS^d: x_{d+1}\le -\sqrt{2}/2 \right\},
\end{equation}
then $f_\mu =0$ is the zero function on $\bB^d$. Following \citet{bach2017breaking} and \citet{yang2024optimal}, we can define a norm on $\cF_\sigma$ by 
\[
\gamma(f):= \inf \left\{\|\mu\|: f=f_\mu, \mu \in \cM(\bS^d) \right\}.
\]
By the compactness of closed balls in $\cM(\bS^d)$ (due to Prokhorov's theorem), the infimum defining $\gamma(f)$ is attained by some measure $\mu$. One can show that $\cF_\sigma$ equipped with the norm $\gamma$ is a Banach space. Since the norm is defined through the variation norm, we call $\cF_\sigma$ the variation space (of shallow ReLU neural networks). For any $M>0$, we denote
\[
\cF_\sigma(M):= \{ f\in \cF_\sigma : \gamma(f)\le M \}.
\]
Each function in this class can be thought of as an infinitely wide neural network with a constraint on its weights. There are several other definitions and characterizations of function spaces corresponding to infinitely wide neural networks in recent studies \citep{bartolucci2023understanding,ongie2020function,parhi2022what,parhi2023minimax,savarese2019how,siegel2020approximation,siegel2022sharp,siegel2023characterization,siegel2023optimal}. We will give more discussions in Section \ref{sec: discussion}.

The neural networks with finite neurons are corresponding to discrete measures $\mu \in \cM(\bS^d)$ that are supported on finitely many points. Let us denote the collection of these measures by $\cM_{disc}(\bS^d)$, which is a linear subspace of $\cM(\bS^d)$, and define the corresponding function class $\cF_{\sigma,disc}:= \{ f_\mu: \mu \in \cM_{disc}(\bS^d) \}$. Similar to the norm $\gamma$ defined on $\cF_\sigma$, we can also define a norm on $\cF_{\sigma,disc}$ by
\[
\kappa(f) := \inf \left\{\|\mu\|: f=f_\mu, \mu \in \cM_{disc}(\bS^d) \right\}.
\]
By definition, $\gamma(f) \le \kappa(f)$ for all $f\in \cF_{\sigma,disc}$. But we do not know whether these two norms are equal for $f\in \cF_{\sigma,disc}$. Nevertheless, it is good enough to bound $\gamma(f)$ by $\kappa(f)$ in our theory. We will see soon that the norm $\kappa(f)$ can be estimated in practice, and hence it can be used in statistical estimations.

For any $N\in \bN$, the function class $\NN(N) \subseteq \cF_{\sigma,disc}$ of shallow neural networks with (at most) $N$ neurons can be parameterized in the form
\begin{equation}\label{parameterization}
f_\theta(x):= \sum_{i=1}^N a_i \sigma((x^\intercal,1)w_i),\quad a_i\in \bR, w_i\in \bR^{d+1}, x\in \bB^d,
\end{equation}
where $\theta$ denotes the vector of parameters
\[
\theta := (a_1,w_1^\intercal,\dots,a_N,w_N^\intercal)^\intercal \in \bR^{(d+2)N}.
\]
Note that, in this parameterization, we do not restrict the internal weights $w_i$ to $\bS^d$. Thanks to the homogeneity of the ReLU function, it is easy to see that $f_\theta =f_\mu$ with discrete measure $\mu= \sum_{i=1}^N c_i \delta_{v_i}\in \cM_{disc}(\bS^d)$, where $v_i= w_i/\|w_i\|_2$ and $c_i= a_i \|w_i\|_2$ if $\|w_i\|_2 \neq 0$, and $c_i=0$ if $\|w_i\|_2 = 0$. This implies that we can estimate $\kappa(f_\theta)$ by the following quantity
\[
\kappa(\theta):= \sum_{i=1}^N |a_i| \|w_i\|_2.
\]
With a slight abuse of notation, we have denoted it by the same notation as the norm $\kappa$. The reason is given by the following theorem, which shows that $\kappa(f_\theta)$ is equivalent to the minimum of $\kappa(\theta)$ over all parameterization (\ref{parameterization}).

\begin{theorem}\label{norm computation}
For any $f\in \NN(N)$,
\[
\kappa(f) \le \inf \left\{ \kappa(\theta): f = f_{\theta}, \theta\in \bR^{(d+2)N} \right\}\le 3\kappa(f).
\]
\end{theorem}

The proof is given in Subsection \ref{sec: proof of norm}. In the proof, we actually show that any $f \in \NN(N)$ can be reduced to the following form
\begin{equation}\label{reduced form}
f(x) = \sum_{k=1}^K c_k \sigma((x^\intercal,1)v_k) + (x^\intercal,1)w,
\end{equation}
where $0\le K\le N$, $w\in \bR^{d+1}$ and $v_1,\dots, v_K\in S_0$ satisfy $v_i \neq \pm v_k$ for $1\le i\neq k\le K$. The set $S_0$ is defined by $S_0 :=\bS^d \setminus (S_- \cup S_+)$, where $S_-$ is defined by (\ref{S_-}) and $S_+:=\{-v:v\in S_-\}$. Note that the set $S_0$ contains vectors $v\in \bS^d$ such that the function $x \mapsto \sigma((x^\intercal,1) v)$ is nonlinear on $\bB^d$. We prove that, for the reduced parameterization (\ref{reduced form}),
\begin{equation}\label{norm equ}
\|w\|_2 \lor \sum_{k=1}^K |c_k| \le \kappa(f) \le 2\|w\|_2 + \sum_{k=1}^K |c_k|,
\end{equation}
where we use the notation $a\lor b:= \max\{a,b\}$. The inequality (\ref{norm equ}) gives a way to estimate $\kappa(f)$ in practice. Roughly speaking, the norm $\kappa(f)$ is a combination of the variation of the nonlinear part of $f$ and the norm of the linear part.

Now, we are ready to introduce the central object of this paper
\[
\NN(N,M) := \left\{ f_\theta\in \NN(N): \kappa(\theta)\le M \right\}.
\]
Since $\gamma(f_\theta)\le \kappa(f_\theta) \le \kappa(\theta)$, we have $\NN(N,M) \subseteq \cF_\sigma(M)$. It was shown in \citet[Proposition 2.2]{yang2024optimal} that $\cF_\sigma(M)$ is the closure of $\cup_{N\in \bN} \NN(N,M)$ in $C(\bB^d)$. The rates of approximation by $\NN(N,M)$ for smooth function classes are also studied in many recent works \citep{bach2017breaking,klusowski2018approximation,siegel2022sharp,mao2023rates,siegel2023optimal,yang2024optimal}. To describe these approximation results, let us first recall the classical notion of smoothness of functions. Given a smoothness index $\alpha>0$, we write $\alpha=r+\beta$ where $r\in \bN_0 :=\bN \cup\{0\}$ and $\beta \in(0,1]$. Let $C^{r,\beta}(\bR^d)$ be the H\"older space with the norm
\[
\|f\|_{C^{r,\beta}(\bR^d)} := \max\left\{ \|f\|_{C^r(\bR^d)}, \max_{\|s\|_1=r}|\partial^s f|_{C^{0,\beta}(\bR^d)} \right\},
\]
where $s=(s_1,\dots,s_d) \in \bN_0^d$ is a multi-index and 
\begin{align*}
\|f\|_{C^r(\bR^d)} &:= \max_{\|s\|_1\le r} \|\partial^s f\|_{L^\infty(\bR^d)}, \\
|f|_{C^{0,\beta}(\bR^d)} &:= \sup_{x\neq y\in \bR^d} \frac{|f(x)-f(y)|}{\|x-y\|_2^\beta}.
\end{align*}
Here, we use $\|\cdot\|_{L^\infty}$ to denote the supremum norm, since we only consider continuous functions in this paper. We write $C^{r,\beta}(\bB^d)$ for the Banach space of all
restrictions to $\bB^d$ of functions in $C^{r,\beta}(\bR^d)$. The norm of this space is given by $\|f\|_{C^{r,\beta}(\bB^d)} = \inf\{ \|g\|_{C^{r,\beta}(\bR^d)}: g\in C^{r,\beta}(\bR^d) \mbox{ and } g=f \mbox{ on } \bB^d\}$. For convenience, we will denote the unit ball of $C^{r,\beta}(\bB^d)$ by 
\[
\cH^\alpha:= \left\{ f\in C^{r,\beta}(\bB^d): \|f\|_{C^{r,\beta}(\bB^d)}\le 1 \right\}.
\]
Note that, for $\alpha=1$, $\cH^\alpha$ is a class of Lipschitz continuous functions.

In the recent work \citep{yang2024optimal}, we characterized how well one can use neural networks in $\cF_\sigma(M)$ to approximate functions in $\cH^\alpha$. It was shown by \citet[Theorem 2.1]{yang2024optimal} that if $\alpha>(d+3)/2$ then $\cH^\alpha \subseteq \cF_\sigma(M)$ for some constant $M$, and if $\alpha<(d+3)/2$ then
\[
\sup_{h\in \cH^\alpha} \inf_{f\in \cF_\sigma(M)} \|h-f\|_{L^\infty(\bB^d)} \lesssim M^{-\frac{2\alpha}{d+3-2\alpha}}.
\]
For the critical value $\alpha=(d+3)/2$, it seems that the inclusion $\cH^\alpha \subseteq \cF_\sigma(M)$ also holds for some constant $M$, but \citet{yang2024optimal} only obtained the exponential approximation rate $\cO(\exp(-\alpha M^2))$. Combining these results with the rate of approximation for $\cF_\sigma(M)$ by its subclass $\cF_\sigma(N,M)$ derived in \citet{bach2017breaking,siegel2023optimal}, one can obtain the following approximation bounds for shallow neural networks. Recall that we use the notation $a\lor b:= \max\{a,b\}$.

\begin{theorem}[\citealt{yang2024optimal}]\label{app}
We have the following approximation bounds.
\begin{enumerate}[label=\textnormal{(\arabic*)},parsep=0pt]
\item For $\cH^\alpha$ with $0<\alpha<(d+3)/2$, it holds that
\[
\sup_{h\in \cH^\alpha} \inf_{f\in \NN(N,M)} \|h-f\|_{L^\infty(\bB^d)} \lesssim N^{-\frac{\alpha}{d}} \lor M^{-\frac{2\alpha}{d+3-2\alpha}}.
\]

\item For $\cF_{\sigma}(1)$, it holds that
\[
\sup_{h \in \cF_{\sigma}(1)} \inf_{f\in \NN(N,1)} \|h-f\|_{L^\infty(\bB^d)} \lesssim N^{- \frac{d+3}{2d}}.
\]
\end{enumerate}
\end{theorem}

As discussed by \citet{yang2024optimal}, for $\cH^\alpha$ with $0<\alpha<(d+3)/2$, the approximation rate is optimal in terms of the number of neurons $N$ and the norm constraint $M$. We will use this result to establish the optimal rates of convergence for shallow neural networks in the nonparametric regression problem.

\section{Nonparametric Regression}\label{sec: main results}

We consider the problem of nonparametric regression in a classical setting. Suppose we have a data set of $n\ge 2$ samples $\cD_n = \{(X_i,Y_i)\}_{i=1}^n$, which are independently and identically generated from a probability distribution supported on $\bB^d\times \bR$. Let $\mu$ be the marginal distribution of the covariate $X$ and $h(x):= \bE[Y|X=x]$ denote the regression function. In this paper, we assume that $|Y|\le B$ with some fixed constant $B>0$ and $h\in \cH$, where we will study the cases that $\cH$ is the unit ball of the variation space $\cF_{\sigma}(1)$ or the H\"older class $\cH^\alpha$ with $\alpha<(d+3)/2$. Note that $\cH^\alpha$ with $\alpha>(d+3)/2$ is included in the first case by scaling, because $\cH^\alpha \subseteq \cF_\sigma(M)$ for some $M$ \citep{yang2024optimal}. Let us denote the noises by
\[
\eta_i := Y_i - h(X_i),\quad i=1,\dots,n,
\]
then $\bE[\eta_i]=0$ and $|\eta_i|\le 2B$. We denote the sample data points by the sequence $X_{1:n} := (X_1,\dots,X_n)$ (similarly, $\eta_{1:n} := (\eta_1,\dots,\eta_n)$ denotes the sequence of noises). The empirical distribution is denoted by $\mu_n:= \frac{1}{n} \sum_{i=1}^n \delta_{X_i}$, and the associated $L^2(\mu_n)$ norm is given by $\|f\|_{L^2(\mu_n)}^2:= \frac{1}{n}\sum_{i=1}^n f(X_i)^2$.

\subsection{Least Squares}

In practice, one popular way to estimate the regression function $h\in \cH$ is by the constrained least squares
\begin{equation}\label{least squares}
\widehat{f}_n \in \argmin_{f\in \cF_n} \frac{1}{n} \sum_{i=1}^n (f(X_i)- Y_i)^2.
\end{equation}
We are interested in the case that $\cF_n = \NN(N_n,M_n)$ is parameterized by a shallow ReLU neural network. For simplicity, we assume here and in the sequel that the minimum above indeed exists. The performance of the estimation is measured by the expected risk
\[
\cL(\widehat{f}_n) := \bE_{(X,Y)} [(\widehat{f}_n(X)-Y)^2].
\]
It is equivalent to evaluating the estimator by the excess risk
\[
\|\widehat{f}_n - h\|_{L^2(\mu)}^2 = \cL(\widehat{f}_n) - \cL(h).
\]

In statistical analysis of learning algorithms, we often require that the hypothesis class is uniformly bounded. We define the truncation operator $\cT_b$ with level $b>0$ for real-valued functions $f$ as
\[
\cT_bf(x) := 
\begin{cases}
f(x), &\quad \mbox{if }|f(x)|\le b, \\
\sgn(f(x)) b, &\quad \mbox{if } |f(x)|> b.
\end{cases}
\]
For a class of real-valued functions $\cF$, we use the notation $\cT_b \cF:=\{\cT_bf:f\in\cF \}$. It is easy to see that, if $b\ge \sup_{h\in\cH} \|h\|_{L^\infty(\bB^d)}$, then $\|\cT_b\widehat{f}_n-h\|_{L^2(\mu)} \le \|\widehat{f}_n-h\|_{L^2(\mu)}$ for any $h\in \cH$. Hence, truncating the output of the estimator $\widehat{f}_n$ appropriately dose not increase the excess risk. In the following, we will simply take $b=B$. Our first result gives convergence rates for the least squares estimator $\widehat{f}_n$ with $\cF_n$ being a shallow neural network.

\begin{theorem}\label{main theorem}
Let $\widehat{f}_n$ be the estimator (\ref{least squares}) with $\cF_n = \NN(N_n,M_n)$.
\begin{enumerate}[label=\textnormal{(\arabic*)},parsep=0pt]
\item If $\cH = \cH^\alpha$ with $\alpha<(d+3)/2$, we choose
\[
M_n \asymp n^{\frac{d+3-2\alpha}{2d+4\alpha}}, \quad N_n \gtrsim n^{\frac{d}{d+2\alpha}},
\]
then the bound
\[
\|\cT_B\widehat{f}_n-h\|_{L^2(\mu)}^2 \lesssim n^{-\frac{2\alpha}{d+2\alpha}} \log n
\]
holds with probability at least $1-c_1 \exp(-c_2 n^{\frac{d}{d+2\alpha}} \log n)$ for some constants $c_1,c_2>0$.

\item If $\cH = \cF_{\sigma}(1)$, we let $M\ge 1$ be a constant and choose
\[
M_n=M, \quad N_n \gtrsim n^{\frac{d}{2d+3}},
\]
then the bound
\[
\|\cT_B\widehat{f}_n-h\|_{L^2(\mu)}^2 \lesssim n^{-\frac{d+3}{2d+3}} \log n
\]
holds with probability at least $1-c_1 \exp(-c_2 n^{\frac{d}{2d+3}} \log n)$ for some constants $c_1,c_2>0$.
\end{enumerate}
\end{theorem}

Note that the (implied) constants in the theorem depend on $d$, $\alpha$ and $B$. Notice that we can always take $N_n \gtrsim n$, which gives justification for the use of over-parameterized neural networks in practice. The results also hold for $\cF_n=\cF_{\sigma}(M_n)$, which corresponds to $N_n=\infty$. The required minimal widths $N_n$ of neural networks match the results of \citet[Theorem 4.2]{yang2024optimal} in the under-parameterized regime. For over-parameterized shallow neural networks, Theorem \ref{main theorem} improves the rates of convergence in \citet[Theorem 4.7]{yang2024optimal}. Ignoring logarithmic factors, the rates $n^{-\frac{2\alpha}{d+2\alpha}}$ for $\cH^\alpha$ and $n^{-\frac{d+3}{2d+3}}$ for $\cF_{\sigma}(1)$ are minimax optimal, as shown by \citet{stone1982optimal,parhi2023minimax,yang2024optimal}. 

Similar to the classical analysis of nonparametric regression \citep{gyoerfi2002distribution,wainwright2019high}, our proof of Theorem \ref{main theorem} bounds the error $\|\cT_B\widehat{f}_n-h\|_{L^2(\mu)}^2$ as the inequality (\ref{error decom}) by two terms: the approximation error $\cE_{app}(\cH, \cF_n)$ and the complexity $\cC_n(\cF_n)$. The approximation error can be handled by Theorem \ref{app}. We control the complexity by using the localization technique \citep{bartlett2005local,koltchinskii2006local}.

\begin{definition}[Local complexity]\label{complexity}
Let $\xi_{1:n}$ be a sequence of independent zero-mean random variables. For a given radius $\delta>0$ and a sequence of sample points $X_{1:n}$, we define the local complexity of a function class $\cF$ at scale $\delta$ with respect to $\xi_{1:n}$ by
\[
\cG_n(\cF;\delta,\xi_{1:n}) := \bE_{\xi_{1:n}} \left[ \sup_{f\in \cF, \|f\|_{L^2(\mu_n)}\le \delta} \left| \frac{1}{n} \sum_{i=1}^n \xi_i f(X_i) \right|\right].
\]
If each $\xi_i$ is the Rademacher random variable (taking values $\pm 1$ with equal probability $1/2$), $\cG_n(\cF;\delta,\xi_{1:n})$ is the local Rademacher complexity and is denoted by $\cR_n(\cF;\delta)$.
\end{definition}

This definition is a direct generalization of the local Rademacher and Gaussian complexities \citep{bartlett2005local,koltchinskii2006local,wainwright2019high}. It is obvious from the definition that $\cG_n(\cF;\delta,\xi_{1:n})$ also depends on the sample points $X_{1:n}$, which is omitted in the notation because we will provide bounds for local complexities that hold uniformly on $X_{1:n}$. When $\xi_{1:n}$ is the noise sequence $\eta_{1:n}$, the local complexity $\cG_n(\cF;\delta,\eta_{1:n})$ measures how well the function class $\cF$ correlates with the noise on the samples $X_{1:n}$. Note that, $\cG_n(\cF;\delta,\eta_{1:n})$ is a data-dependent quantity. If the sample points $X_{1:n}$ are random, then $\cG_n(\cF;\delta,\eta_{1:n})$ is also a random variable. 

To use the local complexities, we often require that the function class $\cF$ is star-shaped (around the origin), meaning that, for any $f\in \cF$ and $a\in [0,1]$, the function $af\in \cF$. If the star-shaped condition fails to hold, one can consider the star hull
\[
\star(\cF) := \{af:f\in \cF,a\in [0,1] \}.
\]
One of the key properties of local complexities of a star-shaped class $\cF$ is that the function $\delta \mapsto \cG_n(\cF;\delta,\xi_{1:n})/\delta$ is non-increasing on $(0,\infty)$. This can be easily proven by observing that the rescaled function $\widetilde{f} = \frac{\delta_1}{\delta_2} f \in \cF$ if $f\in \cF$ and $0<\delta_1\le \delta_2$ (see \citealt[Lemma 13.6]{wainwright2019high} for example). Consequently, for any constant $c>0$, the inequality $\cG_n(\cF;\delta,\xi_{1:n})\le c\delta^2$ always has positive solutions $\delta=\delta_n(\cF)$. Our proof of Theorem \ref{main theorem} uses these solutions to control the complexities of neural networks $\cC_n(\cF_n) \lesssim \delta_n^2(\cF_n)$ (see Lemmas \ref{oracle inequality} and \ref{uniform law} below). Note that the function classes $\cF_{\sigma}(M)$ and $\NN(N,M)$ are star-shaped. The following theorem gives an estimate for the local complexities of shallow neural networks, which holds uniformly for all sample points $X_{1:n}$ and network width.

\begin{theorem}\label{local Gc bound}
Let $\xi_{1:n}$ be a sequence of independent sub-Gaussian random variables with parameter $\varsigma>0$ in the sense that
\[
\bE[\exp(\lambda \xi_i)] \le \exp(\varsigma^2\lambda^2/2), \quad \forall \lambda\in \bR.
\]
Then, for any $0<\delta\le M$,
\[
\cG_n(\cF_{\sigma}(M);\delta,\xi_{1:n}) \lesssim \frac{\varsigma \delta^{\frac{3}{d+3}} M^{\frac{d}{d+3}}}{\sqrt{n}} \sqrt{\log(nM/\delta)},
\]
where the implied constant is independent of $\xi_{1:n}$ and the sample points $X_{1:n}$ in $\bB^d$. The bound also holds for the function class $\star(\cT_B\cF_{\sigma}(M))$ for any constant $B>0$.
\end{theorem}

The proof of Theorem \ref{local Gc bound} is given in Subsection \ref{sec: proof of Gc}. Note that, when $\delta=M$, we obtain the bound $\cO(M \sqrt{\log n}/\sqrt{n})$ which is the same as that for (global) Rademacher complexity derived in \citet[Theorem 3.2]{golowich2020size}, if we ignore the logarithmic factor. The key point of Theorem \ref{local Gc bound} is that the bound is independent of the network width so that we can apply it to over-parameterized neural networks. 

Now, let us come back to the proof of Theorem \ref{main theorem}. We begin with a decomposition of the excess risk $\|\cT_B\widehat{f}_n-h\|_{L^2(\mu)}^2$, using similar ideas as \citet{nakada2020adaptive} and \citet{schmidthieber2020nonparametric} (see also \citet[Chapters 13 and 14]{wainwright2019high}). For any $f\in \cF_n$ (we will take $f$ to minimize the approximation error $\|f-h\|_{L^\infty(\bB^d)})$, observe that 
\begin{align}
\|\cT_B\widehat{f}_n-h\|_{L^2(\mu)}^2 &\le 2\|\cT_B\widehat{f}_n-\cT_Bf\|_{L^2(\mu)}^2 + 2\|\cT_Bf-h\|_{L^2(\mu)}^2 \nonumber \\
&\le 2\|\cT_{2B}(\widehat{f}_n-f)\|_{L^2(\mu)}^2 + 2\|f-h\|_{L^2(\mu)}^2, \label{triangle ineq 1}
\end{align}
where we use $|\cT_B\widehat{f}_n(x)-\cT_Bf(x)| \le |\cT_{2B} (\widehat{f}_n-f)(x)|$ and $\|h\|_{L^\infty(\bB^d)} \le B$ in the second inequality. Note that, for the empirical error,
\begin{equation}\label{triangle ineq 2}
\|\cT_{2B}(\widehat{f}_n-f)\|_{L^2(\mu_n)}^2 \le \|\widehat{f}_n-f\|_{L^2(\mu_n)}^2 \le 2\|\widehat{f}_n-h\|_{L^2(\mu_n)}^2 + 2\|f-h\|_{L^2(\mu_n)}^2. 
\end{equation}
By the definition of $\widehat{f}_n$,
\[
\frac{1}{n} \sum_{i=1}^n (\widehat{f}_n(X_i)- Y_i)^2 \le \frac{1}{n} \sum_{i=1}^n (f(X_i)- Y_i)^2.
\]
Using $Y_i = h(X_i) + \eta_i$, we have the base inequality
\begin{equation}\label{base inequality}
\|\widehat{f}_n-h\|_{L^2(\mu_n)}^2 \le \|f-h\|_{L^2(\mu_n)}^2 + \frac{2}{n} \sum_{i=1}^n \eta_i \left(\widehat{f}_n(X_i)- f(X_i) \right).
\end{equation}
Our proof can be divided into the following three steps.
\begin{enumerate}[leftmargin=60pt, label=\textbf{Step \arabic*.}, parsep=0pt]
\item Estimating $\|\widehat{f}_n-h\|_{L^2(\mu_n)}^2$ by using the base inequality (\ref{base inequality}).

\item Bounding $\|\cT_{2B}(\widehat{f}_n-f)\|_{L^2(\mu)}^2$ by its empirical counterpart $\|\cT_{2B}(\widehat{f}_n-f)\|_{L^2(\mu_n)}^2$.

\item Combining inequalities (\ref{triangle ineq 1}), (\ref{triangle ineq 2}), Step 1, Step 2 with the approximation result (Theorem \ref{app}).
\end{enumerate}

In Step 1, we can treat the sample points $X_{1:n}$ as fixed so that the the randomness is only from the noises $\eta_{1:n}$. The effect of the noise can be measured by the local complexity with respect to the noise. The detail is given in the following lemma, which provides an oracle inequality for the least squares estimator. It is convenient to use the notation $\partial\cF:= \{f_1-f_2:f_1,f_2\in \cF\}$.

\begin{lemma}\label{oracle inequality}
For any fixed sample points $X_{1:n}$, let $\widehat{f}_n$ be the estimator (\ref{least squares}) and $\delta_n$ be any positive solution to the inequality
\[
\cG_n(\star(\partial \cF_n);\delta_n,\eta_{1:n}) \le \delta_n^2.
\]
Then, there are constants $c_1,c_2>0$ such that the bound
\[
\|\widehat{f}_n-h\|_{L^2(\mu_n)}^2 \le 3 \|f-h\|_{L^2(\mu_n)}^2 + 32 \delta_n^2, \quad \forall f\in \cF_n, 
\]
holds with probability at least $1-c_1 \exp(-c_2 n\delta_n^2/B^2)$.
\end{lemma}

This lemma is similar to \citet[Theorem 13.13]{wainwright2019high} which assumes that the noise is standard Gaussian. Our proof is essentially the same as that of \citet[Theorem 13.13]{wainwright2019high}. Since we will use a similar argument in the proof of Lemma \ref{oracle inequality reg}, we give a proof in Subsection \ref{sec: proof of oracle ineq} for completeness. Note that $\delta_n^2$ in Lemma \ref{oracle inequality} corresponds to the complexity $\cC_n(\cF)$ in the decomposition (\ref{error decom}). But Lemma \ref{oracle inequality} only bounds the error in the $L^2(\mu_n)$ norm, rather than the $L^2(\mu)$ norm. So, we need the second step.

In Step 2, we need to quantify the effect of the random sample points $X_{1:n}$. This can be done by using the following uniform law of large number with localization from \citet[Theorem 14.1 and Proposition 14.25]{wainwright2019high}. In the lemma, it is required that the function class $\cF$ is star-shaped and $B$-uniformly bounded for some constant $B>0$, meaning that $\|f\|_{L^\infty(\bB^d)} \le B$ for all $f\in \cF$.

\begin{lemma}\label{uniform law}
Given a star-shaped and $B$-uniformly bounded function class $\cF$, let $\epsilon_n$ be any positive solution to the inequality
\[
\cR_n(\cF;\epsilon_n) \le \frac{\epsilon_n^2}{B}.
\]
Then, there are constants $c_1,c_2>0$ such that the bound  
\[
\|f\|_{L^2(\mu)}^2 \le 2 \|f\|_{L^2(\mu_n)}^2 + \epsilon_n^2, \quad \forall f\in \cF,
\]
holds with probability at least $1-c_1 \exp(-c_2 n\epsilon_n^2/B^2)$.
\end{lemma}

We can now prove Theorem \ref{main theorem}. We will use $c_1,c_2$ to denote constants, which may be different in different bounds.

\bigbreak

\begin{proof}\textbf{of Theorem \ref{main theorem}}

\textbf{Step 1.} We apply Lemma \ref{oracle inequality} to $\cF_n=\NN(N_n,M_n)$. Since $\star(\partial \cF_n) \subseteq \cF_{\sigma}(2M_n)$ and $\eta_{1:n}$ is a sequence of independent sub-Gaussian random variables with parameter $\varsigma=2B$, by Theorem \ref{local Gc bound}, we have
\[
\cG_n(\star(\partial \cF_n);\delta_n,\eta_{1:n}) \lesssim \frac{B\delta_n^{\frac{3}{d+3}} M_n^{\frac{d}{d+3}}}{\sqrt{n}} \sqrt{\log(nM_n/\delta_n)},
\]
for all sample points $X_{1:n}$ in $\bB^d$. Hence, we can choose 
\[
\delta_n \asymp n^{-\frac{d+3}{4d+6}} M_n^{\frac{d}{2d+3}} \sqrt{\log(nM_n)}.
\]
Then, with probability at least $1-c_1 \exp(-c_2 n^{\frac{d}{2d+3}} M_n^{\frac{2d}{2d+3}} \log(nM_n))$, it holds 
\begin{equation}\label{step 1 bound}
\|\widehat{f}_n-h\|_{L^2(\mu_n)}^2 \lesssim \|f-h\|_{L^2(\mu_n)}^2 + n^{-\frac{d+3}{2d+3}} M_n^{\frac{2d}{2d+3}} \log(nM_n), \quad \forall f\in \cF_n.
\end{equation}

\textbf{Step 2.} Observe that 
\[
\cT_{2B}(\widehat{f}_n-f) \in \cT_{2B} \partial\cF_n \subseteq \cT_{2B} \cF_{\sigma}(2M_n) \subseteq \star(\cT_{2B} \cF_{\sigma}(2M_n)).
\]
We can apply Lemma \ref{uniform law} to the star-shaped and $2B$-uniformly bounded function class $\star(\cT_{2B} \cF_{\sigma}(2M_n))$. By Theorem \ref{local Gc bound}, 
\[
\cR_n(\star(\cT_{2B} \cF_{\sigma}(2M_n));\epsilon_n) \lesssim  \frac{\epsilon_n^{\frac{3}{d+3}} M_n^{\frac{d}{d+3}}}{\sqrt{n}} \sqrt{\log(nM_n/\epsilon_n)},
\]
and hence we can choose 
\[
\epsilon_n \asymp n^{-\frac{d+3}{4d+6}} M_n^{\frac{d}{2d+3}} \sqrt{\log(nM_n)}.
\]
Then, with probability at least $1-c_1 \exp(-c_2 n^{\frac{d}{2d+3}} M_n^{\frac{2d}{2d+3}} \log(nM_n))$, it holds 
\begin{equation}\label{step 2 bound}
\|\cT_{2B}(\widehat{f}_n-f)\|_{L^2(\mu)}^2 \lesssim \|\cT_{2B}(\widehat{f}_n-f)\|_{L^2(\mu_n)}^2 + n^{-\frac{d+3}{2d+3}} M_n^{\frac{2d}{2d+3}} \log(nM_n), \quad \forall f\in \cF_n.
\end{equation}

\textbf{Step 3.} Combining inequalities (\ref{triangle ineq 1}), (\ref{triangle ineq 2}), (\ref{step 1 bound}) and (\ref{step 2 bound}), we get
\[
\|\cT_B\widehat{f}_n-h\|_{L^2(\mu)}^2 \lesssim \|f-h\|_{L^2(\mu)}^2 + \|f-h\|_{L^2(\mu_n)}^2 + n^{-\frac{d+3}{2d+3}} M_n^{\frac{2d}{2d+3}} \log(nM_n),
\]
which holds with probability at least $1-c_1 \exp(-c_2 n^{\frac{d}{2d+3}} M_n^{\frac{2d}{2d+3}} \log(nM_n))$ simultaneously for all $f\in \cF_n$. Taking infimum over $f\in \cF_n$ shows 
\begin{equation}\label{temp ineq}
\|\cT_B\widehat{f}_n-h\|_{L^2(\mu)}^2 \lesssim \inf_{f\in\cF_n} \|f-h\|_{L^\infty(\bB^d)}^2 + n^{-\frac{d+3}{2d+3}} M_n^{\frac{2d}{2d+3}} \log(nM_n).
\end{equation}
We continue the proof in two cases.

(1) If $\cH = \cH^\alpha$ with $\alpha<(d+3)/2$, then by Theorem \ref{app}, 
\[
\sup_{h\in \cH} \inf_{f\in\cF_n} \|f-h\|_{L^\infty(\bB^d)}^2 \lesssim N_n^{-\frac{2\alpha}{d}} \lor M_n^{-\frac{4\alpha}{d+3-2\alpha}}.
\]
If $M_n \asymp n^{\frac{d+3-2\alpha}{2d+4\alpha}}$ and $N_n \gtrsim n^{\frac{d}{d+2\alpha}}$, then inequality (\ref{temp ineq}) implies
\[
\|\cT_B\widehat{f}_n-h\|_{L^2(\mu)}^2 \lesssim n^{-\frac{2\alpha}{d+2\alpha}} \log n,
\]
holds with probability at least $1-c_1 \exp(-c_2 n^{\frac{d}{d+2\alpha}} \log n)$.

(2) If $\cH = \cF_{\sigma}(1)$, then by Theorem \ref{app}, for any $M_n=M\ge 1$,
\[
\sup_{h\in \cH} \inf_{f\in\cF_n} \|f-h\|_{L^\infty(\bB^d)}^2 \lesssim N_n^{- \frac{d+3}{d}}.
\]
If $N_n \gtrsim n^{\frac{d}{2d+3}}$, then inequality (\ref{temp ineq}) implies
\[
\|\cT_B\widehat{f}_n-h\|_{L^2(\mu)}^2 \lesssim n^{-\frac{d+3}{2d+3}} \log n,
\]
holds with probability at least $1-c_1 \exp(-c_2 n^{\frac{d}{2d+3}} \log n)$.
\end{proof}

Finally, we make a remark on the boundedness assumption of the noise, which can certainly be weaken.

\begin{remark}
The boundedness of the noises $\eta_{1:n}$ is only used to derive the dimension-free high probability bound (\ref{remark}) in the proof of Lemma \ref{oracle inequality} below. The proof relies on the concentration inequality for convex Lipschitz functions. This argument can also be applied to certain sub-Gaussian noises characterized by \citet[Corollary 5.11]{gozlan2017kantorovich} (which of course includes Gaussian). Hence, similar results as Theorem \ref{main theorem} can be obtained for such kind of noises.
\end{remark}

\subsection{Regularized Least Squares}

The least squares estimate (\ref{least squares}) with $\cF_n=\NN(N_n,M_n)$ is a constrained optimization problem, where we make constraints on the norm $\kappa(f) \le M_n$. One can also estimate the regression function $h$ by solving the corresponding regularized optimization problem 
\begin{equation}\label{regularized least squares 0}
\argmin_{f\in \NN(N_n)} \frac{1}{n} \sum_{i=1}^n (f(X_i)- Y_i)^2 + \lambda_n \kappa(f),
\end{equation}
where $\lambda_n> 0$ is a tunable parameter. Usually, the regularization parameter $\lambda_n$ is chosen to satisfy $\lim_{n\to \infty} \lambda_n=0$. We show in the next theorem that, if $\lambda_n$ is chosen properly, the regularized least squares has the same rate of convergence as the least squares estimator. The result is also true if the norm $\kappa(f)$ in (\ref{regularized least squares 0}) is replaced by $\gamma(f)$.

\begin{theorem}\label{main theorem regularized}
Let $\widehat{f}_{n,\lambda_n}$ be a solution of the optimization problem (\ref{regularized least squares 0}).
\begin{enumerate}[label=\textnormal{(\arabic*)},parsep=0pt]
\item If $\cH = \cH^\alpha$ with $\alpha<(d+3)/2$, we choose
\[
N_n \gtrsim n^{\frac{d}{d+2\alpha}},\quad \lambda_n \asymp n^{-\frac{d+3+2\alpha}{2d+4\alpha}} \log n,
\]
then the bound
\[
\|\cT_B\widehat{f}_{n,\lambda_n}-h\|_{L^2(\mu)}^2 \lesssim n^{-\frac{2\alpha}{d+2\alpha}} \log n
\]
holds with probability at least $1-c_1 \exp(-c_2 n^{\frac{d}{d+2\alpha}} \log n)$ for some constants $c_1,c_2>0$.

\item If $\cH = \cF_{\sigma}(1)$, we choose
\[
N_n \gtrsim n^{\frac{d}{2d+3}},\quad \lambda_n \asymp n^{-\frac{d+3}{2d+3}} \log n,
\]
then the bound
\[
\|\cT_B\widehat{f}_{n,\lambda_n}-h\|_{L^2(\mu)}^2 \lesssim n^{-\frac{d+3}{2d+3}} \log n
\]
holds with probability at least $1-c_1 \exp(-c_2 n^{\frac{d}{2d+3}} \log n)$ for some constants $c_1,c_2>0$.
\end{enumerate}
\end{theorem}

As in the proof of Theorem \ref{main theorem}, we first give a decomposition of the excess risk: for any $f\in \NN(N_n)$,
\begin{equation}\label{triangle ineq 3}
\|\cT_B\widehat{f}_{n,\lambda_n}-h\|_{L^2(\mu)}^2 \le 2\|\cT_{2B}(\widehat{f}_{n,\lambda_n}-f)\|_{L^2(\mu)}^2 + 2\|f-h\|_{L^2(\mu)}^2.
\end{equation}
The empirical counterpart of the first term on the right hand size can be bounded as
\begin{equation}\label{triangle ineq 4}
\|\cT_{2B}(\widehat{f}_{n,\lambda_n}-f)\|_{L^2(\mu_n)}^2 \le 2\|\widehat{f}_{n,\lambda_n}-h\|_{L^2(\mu_n)}^2 + 2\|f-h\|_{L^2(\mu_n)}^2. 
\end{equation}
By the definition of $\widehat{f}_{n,\lambda_n}$, 
\[
\frac{1}{n} \sum_{i=1}^n (\widehat{f}_{n,\lambda_n}(X_i)- Y_i)^2 + \lambda_n \kappa(\widehat{f}_{n,\lambda_n}) \le \frac{1}{n} \sum_{i=1}^n (f(X_i)- Y_i)^2 + \lambda_n \kappa(f).
\]
Using $Y_i = h(X_i) + \eta_i$, we get the base inequality
\begin{equation}\label{base inequality reg}
\|\widehat{f}_{n,\lambda_n}-h\|_{L^2(\mu_n)}^2 + \lambda_n \kappa(\widehat{f}_{n,\lambda_n}) \le \|f-h\|_{L^2(\mu_n)}^2 + \lambda_n \kappa(f) + \frac{2}{n} \sum_{i=1}^n \eta_i \left(\widehat{f}_{n,\lambda_n}(X_i)- f(X_i) \right).
\end{equation}
Similar to the proof of Theorem \ref{main theorem}, our proof strategy contains three steps:
\begin{enumerate}[leftmargin=60pt, label=\textbf{Step \arabic*.}, parsep=0pt]
\item Estimating $\|\widehat{f}_{n,\lambda_n}-h\|_{L^2(\mu_n)}^2$ and $\kappa(\widehat{f}_{n,\lambda_n})$ by using the base inequality (\ref{base inequality reg}).

\item Bounding $\|\cT_{2B}(\widehat{f}_{n,\lambda_n}-f)\|_{L^2(\mu)}^2$ by its empirical counterpart $\|\cT_{2B}(\widehat{f}_{n,\lambda_n}-f)\|_{L^2(\mu_n)}^2$.

\item Combining the above estimates with inequalities (\ref{triangle ineq 3}), (\ref{triangle ineq 4}) and the approximation result (Theorem \ref{app}).
\end{enumerate}

The last two steps can be done in a similar way as Theorem \ref{main theorem}. The main difficulty in proving Theorem \ref{main theorem regularized} is that we do not have an explicit bound on the norm $\kappa(\widehat{f}_{n,\lambda_n})$. In order to apply Theorems \ref{app} and \ref{local Gc bound}, one needs to estimate the norm of the learned function from the optimization problem (\ref{regularized least squares 0}) or equivalently the base inequality (\ref{base inequality reg}). A trivial bound is to choose $f=0$ in (\ref{regularized least squares 0}). Then, the definition of $\widehat{f}_{n,\lambda_n}$ implies
\[
\kappa(\widehat{f}_{n,\lambda_n}) \le \frac{1}{n \lambda_n} \sum_{i=1}^n Y_i^2 \le \frac{B^2}{\lambda_n}.
\]
Unfortunately, this bound is too loose to obtain the optimal rate. To see this, let $f^*$ be a minimizer of the regularized approximation error 
\[
\cE_{\lambda_n} := \inf_{f\in \NN(N_n)} \| f-h\|_{L^2(\mu)}^2 + \lambda_n \kappa(f),
\]
which can be regarded as the expected version of the regularized least squares (\ref{regularized least squares 0}). Then, we know that $\kappa(f^*) \le \cE_{\lambda_n}/\lambda_n$ by definition. Yet, we expect $\widehat{f}_{n,\lambda_n}$ to be a good approximation of $f^*$. So, one would expect that $\kappa(\widehat{f}_{n,\lambda_n}) \lesssim \cE_{\lambda_n}/\lambda_n$, which is always better than the bound  $B^2/\lambda_n$, because $\cE_{\lambda_n} \to 0$ if $\lambda_n$ is chosen properly. This phenomenon is also discussed in detail by \citet{wu2005learning} and \citet[Chapter 7]{steinwart2008support} in the analysis of support vector machines. They also discussed how to use iteration technique to improve the trivial bound. Nonetheless, we overcome the difficulty by using a new oracle inequality for the regularized least squares estimator (see Lemma \ref{oracle inequality reg} below). Our proof shows that one can apply this oracle inequality to obtain a bound close to $\cE_{\lambda_n}/\lambda_n$ .

The following lemma can be viewed as a modification of Lemma \ref{oracle inequality} to regularized estimators. Note that the constant $c_0$ in the lemma will be useful in the proof of Corollary \ref{corollary reg} below. Recall that $\partial\cF:= \{f_1-f_2:f_1,f_2\in \cF\}$.

\begin{lemma}\label{oracle inequality reg}
For any fixed sample points $X_{1:n}$ and constant $c_0\ge 1$, let $\widehat{f}_{n,\lambda_n} \in \NN(N_n)$ be any estimator (i.e. a function depending on the random noises $\eta_{1:n}$, where $\bE[\eta_i]=0$ and $|\eta_i|\le 2B$) that satisfies the inequality
\begin{equation}\label{base inequality reg c_0}
\begin{aligned}
&\|\widehat{f}_{n,\lambda_n}-h\|_{L^2(\mu_n)}^2 + \lambda_n \kappa(\widehat{f}_{n,\lambda_n}) \\
&\le c_0 \left(\|f-h\|_{L^2(\mu_n)}^2 + \lambda_n \kappa(f) \right) + \frac{2}{n}  \left|\sum_{i=1}^n \eta_i \left(\widehat{f}_{n,\lambda_n}(X_i)- f(X_i) \right)\right|, \quad \forall f\in \NN(N_n).
\end{aligned}
\end{equation}
For a user-defined parameter $R>0$, let $\delta_n = \delta_n(R)$ be any positive solution to the inequality
\[
\cG_n(\star(\partial \cF_{n,R});\delta_n,\eta_{1:n}) \le \delta_n^2,
\]
where $\cF_{n,R} := \{f\in \NN(N_n): \kappa(f)\le R \}$. There are constants $c_1,c_2>0$ such that, if $\lambda_n \ge 8\delta_n^2/R$, then the bound
\[
\|\widehat{f}_{n,\lambda_n}-h\|_{L^2(\mu_n)}^2 + \lambda_n \kappa(\widehat{f}_{n,\lambda_n}) \le (1+2c_0) \left(\|f-h\|_{L^2(\mu_n)}^2 + \lambda_n \kappa(f)\right) + 64 \delta_n^2, \quad \forall f\in \cF_{n,R},
\]
holds with probability at least $1-c_1 \exp(-c_2 n\delta_n^2/B^2)$.
\end{lemma}

The proof of Lemma \ref{oracle inequality reg} is given in Subsection \ref{sec: proof of oracle ineq reg}. The inequality we obtained in this lemma is in a similar form as the bound (\ref{error decom}), where the complexity is measured by $\delta_n^2$ and the approximation in $L^2$ norm is replaced by the regularized approximation in the empirical norm.

\bigbreak

\begin{proof}\textbf{of Theorem \ref{main theorem regularized}}

To apply Lemma \ref{oracle inequality reg}, let us first compute $\delta_n=\delta_n(R)$, where $R>0$ will be chosen later. Since $\star(\partial \cF_{n,R}) \subseteq \cF_{\sigma}(2R)$, by Theorem \ref{local Gc bound},
\[
\cG_n(\star(\partial \cF_{n,R});\delta_n,\eta_{1:n}) \lesssim \frac{B\delta_n^{\frac{3}{d+3}} R^{\frac{d}{d+3}}}{\sqrt{n}} \sqrt{\log(nR/\delta_n)},
\]
holds for all sample points $X_{1:n}$ in $\bB^d$. Hence, we can choose 
\[
\delta_n^2 \asymp n^{-\frac{d+3}{2d+3}} R^{\frac{2d}{2d+3}} \log(nR).
\]
We denote the approximation error of $\cF_{n,R}$ by
\[
\cE(R) := \inf_{f\in \cF_{n,R}} \|f-h\|_{L^\infty(\bB^d)}^2.
\]
Note that we use the supremum norm instead of $L^2(\mu_n)$ norm in the definition. Lemma \ref{oracle inequality reg} with $c_0=1$ shows that, if $\lambda_n = 8\delta_n^2/R$, then with probability at least $1-c_1 \exp(-c_2 n\delta_n^2)$, it holds 
\[
\|\widehat{f}_{n,\lambda_n}-h\|_{L^2(\mu_n)}^2 + \lambda_n \kappa(\widehat{f}_{n,\lambda_n}) \lesssim \cE(R) + \delta_n^2,
\]
In particular, we have
\[
\kappa(\widehat{f}_{n,\lambda_n}) \lesssim \frac{\cE(R) + \delta_n^2}{\lambda_n}  \lesssim \frac{\cE(R)}{\lambda_n} + R,
\]
where we have used $\delta_n^2 \lesssim \lambda_n R$ by the choice of $\lambda_n$. 

(1) If $\cH = \cH^\alpha$ with $\alpha<(d+3)/2$, then by Theorem \ref{app}, 
\[
\cE(R) \le \sup_{h\in \cH} \inf_{f\in\cF_{n,R}} \|f-h\|_{L^\infty(\bB^d)}^2 \lesssim N_n^{-\frac{2\alpha}{d}} \lor R^{-\frac{4\alpha}{d+3-2\alpha}}.
\]
We choose 
\[
R \asymp n^{\frac{d+3-2\alpha}{2d+4\alpha}}, \quad N_n \gtrsim R^{\frac{2d}{d+3-2\alpha}} \gtrsim n^{\frac{d}{d+2\alpha}},
\]
then 
\begin{equation}\label{estimate 1}
\cE(R)\lesssim n^{-\frac{2\alpha}{d+2\alpha}}, \quad \delta_n^2 \asymp n^{-\frac{2\alpha}{d+2\alpha}} \log n, \quad \lambda_n = 8 \delta_n^2/R \asymp n^{-\frac{d+3+2\alpha}{2d+4\alpha}} \log n.
\end{equation}

(2) If $\cH = \cF_{\sigma}(1)$, then by Theorem \ref{app}, for any $R\ge 1$,
\[
\cE(R) \le \sup_{h\in \cH} \inf_{f\in\cF_{n,R}} \|f-h\|_{L^\infty(\bB^d)}^2 \lesssim N_n^{- \frac{d+3}{d}}.
\]
We choose 
\[
R = 1, \quad N_n \gtrsim n^{\frac{d}{2d+3}},
\]
then 
\begin{equation}\label{estimate 2}
\cE(R)\lesssim n^{-\frac{d+3}{2d+3}}, \quad \delta_n^2 \asymp n^{-\frac{d+3}{2d+3}} \log n, \quad \lambda_n = 8 \delta_n^2/R \asymp n^{-\frac{d+3}{2d+3}} \log n.
\end{equation}
 
Notice that, in both cases, we always have $\cE(R) \lesssim \lambda_n R$, which implies $M_n:=\kappa(\widehat{f}_{n,\lambda_n}) \lesssim R$. Similar to step 2 in the proof of Theorem \ref{main theorem}, we can apply Lemma \ref{uniform law} to the function class $\star(\cT_{2B} \cF_{\sigma}(M_n+R))$ to show that
\[
\|\cT_{2B}(\widehat{f}_{n,\lambda_n}-f)\|_{L^2(\mu)}^2 \lesssim \|\cT_{2B}(\widehat{f}_{n,\lambda_n}-f)\|_{L^2(\mu_n)}^2 + \epsilon_n^2, \quad \forall f\in \cF_{n,R},
\]
holds with probability at least $1-c_1 \exp(-c_2 n\epsilon_n^2)$, where
\[
\epsilon_n^2 \asymp n^{-\frac{d+3}{2d+3}} R^{\frac{2d}{2d+3}} \log(nR) \asymp \delta_n^2.
\]
Combining this with inequalities (\ref{triangle ineq 3}) and (\ref{triangle ineq 4}), where we choose $f \in \cF_{n,R}$ to satisfy the bound $\|f-h\|_{L^\infty(\bB^d)}^2 \lesssim \cE(R)$, we obtain
\[
\|\cT_B\widehat{f}_{n,\lambda_n}-h\|_{L^2(\mu)}^2 \lesssim \cE(R) + \delta_n^2,
\]
holds with probability at least $1-c_1 \exp(-c_2 n\delta_n^2)$. The estimates (\ref{estimate 1}) and (\ref{estimate 2}) then give the desired results.
\end{proof}

Next, we show that it is possible to simplify the optimization problem (\ref{regularized least squares 0}) without worsening the convergence rates. In practice, the neural network $f_\theta\in \NN(N_n)$ is parameterized by $\theta$ in the form (\ref{parameterization}). By Theorem \ref{norm computation}, we can estimate $\kappa(f)$ by $\kappa(\theta)$, which reduces the optimization problem (\ref{regularized least squares 0}) to  
\begin{equation}\label{regularized least squares 1}
\argmin_{\theta\in \bR^{(d+2)N_n}} \frac{1}{n} \sum_{i=1}^n (f_\theta(X_i)- Y_i)^2 + \lambda_n \kappa(\theta).
\end{equation}
One can further simplify the optimization problem by observing that the regularizer $\kappa(\theta) \le \|\theta\|_2^2/2$. It may be easier to solve the following optimization problem 
\begin{equation}\label{regularized least squares 2}
\argmin_{\theta\in \bR^{(d+2)N_n}} \frac{1}{n} \sum_{i=1}^n (f_\theta(X_i)- Y_i)^2 + \frac{\lambda_n}{2} \|\theta\|_2^2.
\end{equation}
The next proposition shows that the two optimization problems (\ref{regularized least squares 1}) and (\ref{regularized least squares 2}) essentially have the same solutions. To the best of our knowledge, this equivalence was first pointed out by \citet[Theorem 1]{neyshabur2015search}.

\begin{proposition}\label{opt equ}
If $\theta^*$ is a solution of the optimization problem (\ref{regularized least squares 1}), then there exits a solution $\widetilde{\theta}^*$ of (\ref{regularized least squares 2}) such that $f_{\theta^*} = f_{\widetilde{\theta}^*}$. Any solution of (\ref{regularized least squares 2}) is also a solution of (\ref{regularized least squares 1}).
\end{proposition}
\begin{proof}
Let us denote $N=N_n$ for simplicity. Due to the positive homogeneity of the ReLU function, we can always rescale the parameter $\theta=(a_1,w_1^\intercal,\dots,a_N,w_N^\intercal)^\intercal$ to the parameter $(c_1a_1,c_1^{-1}w_1^\intercal,\dots,c_Na_N,c_N^{-1}w_N^\intercal)^\intercal$ with $c_i>0$, without altering the function $f_\theta$. We use $r(\theta) = (\widetilde{a}_1,\widetilde{w}_1^\intercal,\dots,\widetilde{a}_N,\widetilde{w}_N^\intercal)^\intercal$ to denote the rescaled parameter of $\theta$ that satisfies $|\widetilde{a}_i|=\|\widetilde{w}_i\|_2$ for all $1\le i\le N$ (if $|a_i|\|w_i\|_2=0$, we let $|\widetilde{a}_i|=\|\widetilde{w}_i\|_2=0$). Note that we always have $\kappa(\theta) = \kappa(r(\theta)) = \|r(\theta)\|_2^2/2 \le \|\theta\|_2^2/2$.

Let us denote the empirical square loss of $f_\theta$ by $\cL_n(f_\theta)=\frac{1}{n} \sum_{i=1}^n (f_\theta(X_i)- Y_i)^2$ for convenience. If $\theta^*$ is a solution of (\ref{regularized least squares 1}), then $r(\theta^*)$ is also a solution of (\ref{regularized least squares 1}). Thus, for any $\theta$, we have 
\begin{align*}
\cL_n(f_{r(\theta^*)}) + \frac{\lambda_n}{2} \|r(\theta^*)\|_2^2 &= \cL_n(f_{r(\theta^*)}) + \lambda_n \kappa(r(\theta^*)) \\
&\le \cL_n(f_{\theta}) + \lambda_n \kappa(\theta) \\
&\le \cL_n(f_{\theta}) + \frac{\lambda_n}{2} \|\theta\|_2^2,
\end{align*}
which shows that $r(\theta^*)$ is a solution of (\ref{regularized least squares 2}). Conversely, if $\theta^*$ is a solution of (\ref{regularized least squares 2}), then $r(\theta^*) = \theta^*$ because $\|r(\theta)\|_2^2 \le \|\theta\|_2^2$. Thus, for any $\theta$,
\begin{align*}
\cL_n(f_{\theta^*}) + \lambda_n \kappa(\theta^*) &= \cL_n(f_{\theta^*}) + \frac{\lambda_n}{2} \|\theta^*\|_2^2 \\
&\le \cL_n(f_{r(\theta)}) + \frac{\lambda_n}{2} \|r(\theta)\|_2^2 \\
&= \cL_n(f_{\theta}) + \kappa(\theta),
\end{align*}
which shows that $\theta^*$ also minimizes (\ref{regularized least squares 1}).
\end{proof}

Using the same proof idea as Theorem \ref{main theorem regularized}, one can show that the solutions of the regularized optimization problems (\ref{regularized least squares 1}) and (\ref{regularized least squares 2}) can also achieve the minimax optimal rates for learning functions in $\cH^\alpha$ or $\cF_\sigma(1)$. The fundamental reason is that these estimators satisfy the base inequality (\ref{base inequality reg}) up to constant factors.

\begin{corollary}\label{corollary reg}
Let $\widehat{f}_{n,\lambda_n} = f_{\widehat{\theta}_{n,\lambda_n}}$, where $\widehat{\theta}_{n,\lambda_n}$ is a solution of the optimization problem (\ref{regularized least squares 1}) or (\ref{regularized least squares 2}). Then the conclusion of Theorem \ref{main theorem regularized} also holds.
\end{corollary}
\begin{proof}
By Proposition \ref{opt equ}, it is enough to consider the case that $\widehat{\theta}_{n,\lambda_n}$ is a solution of the optimization problem (\ref{regularized least squares 1}). As in the proof of Theorem \ref{main theorem regularized}, the inequalities (\ref{triangle ineq 3}) and (\ref{triangle ineq 4}) still hold for $ \widehat{f}_{n,\lambda_n}$. The only difference is the base inequality
\[
\|\widehat{f}_{n,\lambda_n}-h\|_{L^2(\mu_n)}^2 + \lambda_n \kappa(\widehat{\theta}_{n,\lambda_n}) \le \|f-h\|_{L^2(\mu_n)}^2 + \lambda_n \kappa(\theta) + \frac{2}{n} \sum_{i=1}^n \eta_i \left(\widehat{f}_{n,\lambda_n}(X_i)- f(X_i) \right),
\]
where $\theta \in \bR^{(d+2)N_n}$ satisfy $f_\theta =f\in \NN(N_n)$. 
By Theorem \ref{norm computation}, we can choose some $\theta$ such that $\kappa(\theta) \le 3\kappa(f)$. Hence,
\[
\|\widehat{f}_{n,\lambda_n}-h\|_{L^2(\mu_n)}^2 + \lambda_n \kappa(\widehat{f}_{n,\lambda_n}) \le 3 \left(\|f-h\|_{L^2(\mu_n)}^2 + \lambda_n \kappa(f)\right) + \frac{2}{n} \sum_{i=1}^n \eta_i \left(\widehat{f}_{n,\lambda_n}(X_i)- f(X_i) \right),
\]
where we also use $\kappa(\widehat{f}_{n,\lambda_n}) \le \kappa(\widehat{\theta}_{n,\lambda_n})$. Therefore, we can apply Lemma \ref{oracle inequality reg} with $c_0=3$. The remained proof is the same as Theorem \ref{main theorem regularized}.
\end{proof}

\section{Discussions and Related Works}\label{sec: discussion}

\textit{Infinitely wide neural networks.} The seminal work of \citet{barron1993universal} obtained dimension independent rate of approximation by shallow neural networks (with sigmoidal activations) for functions $h$, whose Fourier transform $\widehat{h}$ satisfies $\int_{\bR^d} \|\omega\||\widehat{h}(\omega)|d\omega<\infty$. Barron's results have been improved and extended to functions of certain integral forms, such as (\ref{integral form}), and ReLU activation in recent years \citep{klusowski2018approximation,makovoz1996random,mhaskar2004tractability,mhaskar2020dimension,siegel2020approximation}. These results lead to the study of variation spaces corresponding to neural networks \citep{bach2017breaking,siegel2023characterization,siegel2023optimal,yang2024optimal}, which are also called Barron spaces in \citet{weinan2019priori,weinan2022barron}. We refer the reader to \citet{siegel2022sharp} for detail discussions on these spaces from an approximation theory perspective.

There is another line of works \citep{bartolucci2023understanding,ongie2020function,parhi2021banach,parhi2022what,parhi2023minimax,savarese2019how} trying to characterize infinitely wide neural networks from a functional analysis point of view. These papers derived representer theorems showing that (finitely wide) neural networks are sparse solutions to data fitting problems with total variation regularization in the Radon domain:
\begin{equation}\label{opt RTV}
\argmin_f \frac{1}{n} \sum_{i=1}^n \ell(Y_i,f(X_i)) + \lambda \RTV(f),
\end{equation}
where the loss function $\ell$ is convex and coercive in its second argument and $\RTV$ is a seminorm defined through the Radon transform. Note that the null space of the seminorm $\RTV$ is the space of affine functions \citep[Lemma 19]{parhi2021banach}. For $f\in \NN(N)$ in the reduced form (\ref{reduced form}), $\RTV(f)$ is exactly the variation $\sum_{k=1}^K |c_k|$ by \citet[Lemma 25]{parhi2021banach}, and hence $\kappa(f) \approx \RTV(f) + \|\mbox{linear part of } f\|$ by inequality (\ref{norm equ}). Thus, for the square loss $\ell$, we can view the optimization problem (\ref{opt RTV}) as the problem (\ref{regularized least squares 0}) with width $N_n=\infty$.

Infinitely wide neural networks are also used in the study of training dynamics of neural networks. The neural tangent kernel theory demonstrates that the evolution of a neural network during gradient descent training can be described by a kernel in the infinite width limit \citep{allenzhu2019convergence,du2019gradient,jacot2018neural,oymak2020towards}. The mean field analysis approximates the evolution of the network weights by a gradient flow, defined through a partial differential equation, in the Wasserstein space of probability distributions \citep{mei2018mean,mei2019mean,sirignano2020mean}. These theories established the convergences of gradient descent training of over-parameterized neural networks in certain scalings of network weights \citep{chizat2019lazy}. It would be interesting to see whether these techniques can be applied to study the optimization problems (\ref{least squares}) and (\ref{regularized least squares 1}).

\textit{Nonparametric regression using neural networks.} As discussed in the introduction, rates of convergence of neural network regression estimators have been analyzed by many authors. Minimax optimal rates have been established for sparse deep neural networks \citep{schmidthieber2020nonparametric,suzuki2019adaptivity}, for fully connected deep neural networks \citep{kohler2021rate} and for shallow neural networks \citep{yang2024optimal}. It has also been shown that deep neural networks are able to circumvent the curse of dimensionality under certain conditions, for example, when the intrinsic dimension is low \citep{chen2022nonparametric,jiao2023deep,nakada2020adaptive} or the regression function has certain hierarchical structures \citep{bauer2019deep,schmidthieber2020nonparametric,kohler2021rate}. In the over-parameterized regime, \citet{yang2024optimal} proved nearly optimal rates by using (deep or shallow) neural networks with weight constraints. The paper of \citet{parhi2023minimax} established the minimax optimal rate for the variation space $\cF_{\sigma}$ using infinitely wide shallow neural networks. Our results can be viewed as generalizations of this result to H\"older classes $\cH^\alpha$ and finitely wide networks. Note that all these results are established for least squares, where either the network architecture or the network weights are properly constrained. For regularized least squares, the recent work \citep{zhang2023deep} proved convergence rates for deep neural networks with a regularization similar to this paper. The rates they obtained are asymptotically optimal as the depth goes to infinity. But they did not show how to choose the regularization parameter. Our results show that regularized least squares with shallow networks can achieve the minimax optimal rates, at least for function classes with low smoothness. Our proof method (see Lemma \ref{oracle inequality reg}) may be used to study other regularized estimators. A particular interesting direction is to generalize our results to deep neural networks.

\textit{Interpolation and regularization.} The recent progress in deep learning revealed that over-parameterized models can interpolate noisy data and yet generalize well \citep{belkin2019reconciling,zhang2017understanding}. Motivated by this phenomenon, which is called benign overfitting, there has recently been a line of works on understanding theoretical mechanisms for the good generalization performance of interpolating models (see \citealt{bartlett2021deep,belkin2021fit} for reviews). In particular, benign overfitting has been theoretically established for many models, such as Nadaraya-Watson smoothing estimator \citep{belkin2019does}, linear regression \citep{bartlett2020benign,hastie2022surprises}, kernel regression \citep{liang2020just} and random features models \citep{mei2022generalization}. However, this certainly does not mean that any interpolator can generalize well, simply because there are too many estimators that interpolate observed data. Hence, explicit or implicit regularization is necessary to obtain estimators with good generalization.

In fact, most works on benign overfitting studied the minimal norm interpolations. In our setting, this corresponds to the solutions of 
\begin{equation}\label{minimal norm int}
\argmin_{f\in \NN(N_n)} \kappa(f), \quad s.t. \quad f(X_i)=Y_i,\ i=1,\dots,n,
\end{equation}
which can be viewed as the limiting case of the optimization (\ref{regularized least squares 0}) with $\lambda_n=0$. The setting of over-parameterized neural networks allows us to make an explicit comparison between interpolating and regularized solutions. We think it is a promising research direction to study rates of convergence for the minimal norm interpolation (\ref{minimal norm int}) and compare them with our results. Note that, for linear regression, \citet{hastie2022surprises} showed that optimally regularized ridge regression dominates the minimum $\ell^2$-norm interpolation in test performances. It is reasonable to expect that such a phenomenon also occurs for neural networks (since the regularized least squares (\ref{regularized least squares 0}) achieves the minimax rates). On the other hand, some kernel interpolations have been proven to be inconsistent in low dimensions \citep{buchholz2022kernel,rakhlin2019consistency}. It is an interesting problem to determine whether the minimal norm solutions of (\ref{minimal norm int}) are consistent and minimax optimal.

\section{Proofs}\label{sec: proofs}

This section gives the proofs of Theorem \ref{norm computation}, Theorem \ref{local Gc bound}, Lemma \ref{oracle inequality} and Lemma \ref{oracle inequality reg}.

\subsection{Proof of Theorem \ref{norm computation}} \label{sec: proof of norm}

Recall that $S_-$ is defined by (\ref{S_-}), $S_+=-S_-$ and $S_0 =\bS^d \setminus (S_-\cup S_+)$. Notice that $S_-$ is the set of vectors $v\in \bS^d$ such that $(x^\intercal,1) v \le 0$ for all $x\in \bB^d$. Thus, $v\in \bS^d$ is in $S_+$ if and only if $(x^\intercal,1) v \ge 0$ holds for all $x\in \bB^d$. And, for any $v\in S_0$, $\sigma((x^\intercal,1) v)$ is a nonlinear function on $\bB^d$. Let us first characterize the measures $\tau\in\cM_{disc}(\bS^d)$ such that $f_\tau=0$.

\begin{lemma}\label{zero disc measures}
Let $\tau=\sum_{i=1}^N a_i\delta_{u_i} \in\cM_{disc}(\bS^d)$, where $u_1,\dots,u_N \in \bS^d$ are distinct. Then, $f_\tau(x) = \sum_{i=1}^N a_i \sigma((x^\intercal,1)u_i) =0$ for all $x\in \bB^d$ if and only if the following two conditions hold
\begin{enumerate}[label=\textnormal{(\arabic*)},parsep=0pt]
\item For any $u_i\in S_0$, there exists $j\neq i$ such that $u_j = -u_i$ and $a_j = -a_i$.

\item $\frac{1}{2} \sum_{u_i \in S_0} a_iu_i + \sum_{u_i \in S_+} a_iu_i = 0$.
\end{enumerate}
\end{lemma}
\begin{proof}
Without loss of generality, we assume that $u_i\in S_0\cup S_+$ and $a_i\neq 0$ for $i=1,\dots,N$. Notice that $a_i(x^\intercal,1)u_i = a_i\sigma((x^\intercal,1)u_i) -  a_i\sigma(-(x^\intercal,1)u_i)$. If $\tau$ satisfies condition (1), then
\[
\sum_{u_i \in S_0} a_i (x^\intercal,1)u_i = 2 \sum_{u_i \in S_0} a_i \sigma((x^\intercal,1)u_i).
\]
If, in addition, $\tau$ satisfies condition (2), then
\begin{align*}
f_\tau(x) &= \sum_{u_i \in S_0} a_i \sigma((x^\intercal,1)u_i) + \sum_{u_i \in S_+} a_i (x^\intercal,1)u_i \\
&= \frac{1}{2} \sum_{u_i \in S_0} a_i (x^\intercal,1)u_i + \sum_{u_i \in S_+} a_i (x^\intercal,1)u_i = 0.
\end{align*}

Conversely, assume $f_\tau =0$ on $\bB^d$. Notice that $\sigma((x^\intercal,1)u_i)$ is a piecewise linear function. We denote its ``break points'' by $A_i = \{x: \|x\|_2<1, (x^\intercal,1)u_i =0 \}$. One can check that $A_i$ is non-empty if and only if $u_i\in S_0$. For any $u_i \in S_0$, let $j$ be the index such that $u_j=-u_i$ (it may not exist), then we can find a point $x_i \in A_i = A_j$ such that $x_i \notin A_k$ for any $k\neq i,j$. Consider the decomposition
\[
0= f_\tau(x) = a_i \sigma((x^\intercal,1)u_i) + a_j \sigma((x^\intercal,1)u_j) + \sum_{k\neq i,j} a_k \sigma((x^\intercal,1)u_k).
\]
If $u_j$ does not exist, then the first term is nonlinear in a neighborhood of $x_i$ and the remain terms are linear in this neighborhood, which is a contradiction. Hence, $u_j$ exists and the summation of the first two terms is linear, which implies $a_j=-a_i$ and condition (1) holds. As a consequence, 
\begin{align*}
0= f_\tau(x) &= \frac{1}{2} \sum_{u_i \in S_0} a_i (x^\intercal,1)u_i + \sum_{u_i \in S_+} a_i (x^\intercal,1)u_i \\
&= (x^\intercal,1) \left(\frac{1}{2} \sum_{u_i \in S_0} a_iu_i + \sum_{u_i \in S_+} a_iu_i \right)
\end{align*}
holds for all $x\in \bB^d$. Thus, $\frac{1}{2} \sum_{u_i \in S_0} a_iu_i + \sum_{u_i \in S_+} a_iu_i = 0$, which proves condition (2).
\end{proof}

We now give a proof of Theorem \ref{norm computation}. For any $f\in \NN(N)$, we can always represent it as
\begin{equation}\label{temp form 1}
f(x) = \sum_{\widetilde{v}_i \in S_0} \widetilde{c}_i \sigma((x^\intercal,1)\widetilde{v}_i) + \sum_{\widetilde{v}_i \in S_+} \widetilde{c}_i (x^\intercal,1)\widetilde{v}_i.
\end{equation}
Without loss of generality, we assume that the $\widetilde{v}_i$ are distinct. If there exist $\widetilde{v}_i$ and $\widetilde{v}_j$ satisfying $\widetilde{v}_j = - \widetilde{v}_i \in S_0$, then 
\[
\widetilde{c}_i \sigma((x^\intercal,1)\widetilde{v}_i) + \widetilde{c}_j \sigma((x^\intercal,1)\widetilde{v}_j) = (\widetilde{c}_i + \widetilde{c}_j) \sigma((x^\intercal,1)\widetilde{v}_i) - \widetilde{c}_j (x^\intercal,1)\widetilde{v}_i.
\]
Using this equality, we can reduce (\ref{temp form 1}) to the following form
\begin{equation}\label{temp form 2}
f(x) = \sum_{k=1}^K c_k \sigma((x^\intercal,1)v_k) + (x^\intercal,1)w,
\end{equation}
where $0\le K\le N$, $c_1,\dots, c_K \neq 0$, $w\in \bR^{d+1}$ and $v_1,\dots, v_K\in S_0$ satisfy $v_i \neq \pm v_k$ for $1\le i\neq k\le K$. Note that, if $K=N$, then $w=0$. We are going to construct a measure $\mu\in \cM_{disc}(\bS^d)$ supported on at most $N$ points in $S_0\cup S_+$ such that $f_\mu=f$ and estimate $\kappa(f)$ by $\|\mu\|$. By Lemma \ref{zero disc measures}, any measure $\nu$ satisfies $f_\nu =f_\mu=f$ if and only if $\nu = \mu+\tau$ for some $\tau$ satisfying the two conditions in Lemma \ref{zero disc measures}. Hence, $\kappa(f) = \inf_{\tau} \|\mu+\tau\|$. Without loss of generality, we also assume that $\tau$ is supported on $S_0\cup S_+$. By Lemma \ref{zero disc measures}, we can decompose $\tau$ as
\begin{align*}
\tau &= \sum_{k=1}^K a_k \left(\delta_{v_k} - \delta_{-v_k}\right) + \sum_{j=1}^{J_1} b_j \delta_{u_j} + \sum_{j=J_1+1}^J b_j (\delta_{u_j} - \delta_{-u_j}) \\
&=: \tau_1 + \tau_2 + \tau_3,
\end{align*}
where $u_j\in S_+$ for $1\le j\le J_1$, and $u_j\in S_0\setminus \{\pm v_1,\dots,\pm v_K\}$ for $J_1+1\le j\le J$ satisfies $u_j \neq \pm u_i$ for any $i\neq j$ and
\begin{equation}\label{temp linear dep}
\sum_{k=1}^K a_k v_k + \sum_{j=1}^J b_j u_j =0.
\end{equation}
We continue the proof in four different cases.

\textbf{Case 1}: $w=0$. We let $\mu = \sum_{k=1}^K c_k \delta_{v_k}$, then 
\begin{align*}
\|\mu+\tau\| &= \|\mu +\tau_1\| + \|\tau_2+\tau_3\| \ge \|\mu+\tau_1\| \\
&= \sum_{k=1}^K \left(|c_k+a_k| + |a_k|\right) \ge \sum_{k=1}^K |c_k| = \|\mu\|,
\end{align*}
which implies $\kappa(f) = \|\mu\|$.

\textbf{Case 2}: $w/\|w\|_2 \in S_+ \cup S_-$. We let $\mu = \sum_{k=1}^K c_k \delta_{v_k} + c_{K+1} \delta_{v_{K+1}} =: \mu_1 +\mu_2$, where, if $w/\|w\|_2 \in S_+$, we take $v_{K+1} = w/\|w\|_2 \in S_+$ and $c_{K+1} = \|w\|_2$; otherwise, we take $v_{K+1} = -w/\|w\|_2\in S_+$ and $c_{K+1} = -\|w\|_2$. Since $w\neq 0$, $K\le N-1$ and $\mu$ is supported on at most $K+1\le N$ points. Let $m \in \{1,\dots,J_1\}$ be the index such that $u_m = v_{K+1}$. (If such $m$ does not exists, we set $m=0$ and $b_m=0$ in the following.) Similar to Case 1, we have 
\begin{align*}
\|\mu+\tau\| &= \|\mu_1 +\tau_1\| + \|\mu_2+b_m \delta_{u_m}\| + \|\tau_2 - b_m \delta_{u_m} + \tau_3\| \\
&\ge \sum_{k=1}^K \left(|c_k+a_k| + |a_k|\right) + |c_{K+1} +b_m| + \sum_{j\neq m} |b_j|.
\end{align*}
By equality (\ref{temp linear dep}), 
\[
w = c_{K+1} v_{K+1} =  \sum_{k=1}^K a_k v_k + (c_{K+1}+b_m) u_m + \sum_{j\neq m} b_j u_j,
\]
which implies 
\[
\|\mu+\tau\| \ge \sum_{k=1}^K |a_k| + |c_{K+1} +b_m| + \sum_{j\neq m} |b_j| \ge \|w\|_2.
\]
Combining with the bound $\|\mu+\tau\| \ge \sum_{k=1}^K |c_k|$, we conclude that
\begin{equation}\label{kappa lower bound}
\kappa(f) = \inf_{\tau} \|\mu+\tau\| \ge \|w\|_2 \lor \sum_{k=1}^K |c_k|.
\end{equation}
Therefore,
\[
\kappa(f) \ge \frac{1}{2} \sum_{k=1}^K |c_k| + \frac{1}{2} \|w\|_2 = \frac{1}{2} \|\mu\|.
\]

\textbf{Case 3}: $w=cv_i$ for some $1\le i \le K$ and $c\neq 0$. We let $\mu = \sum_{k=1}^K c_k \delta_{v_k} +c(\delta_{v_i}-\delta_{-v_i})$, which is supported on at most $K+1\le N$ points, then $\|\mu\| = \sum_{k\neq i} |c_k| + |c_i+c| + |c|$ and 
\begin{align*}
\|\mu+\tau\| &= \|\mu + \tau_1\| + \|\tau_2+ \tau_3\| \\
&\ge \sum_{k\neq i} \left(|c_k+a_k| + |a_k|\right) + |c_i+c+a_i| + |c+a_i| + \sum_{j=1}^J |b_j|.
\end{align*}
Similar to Case 2, by equality (\ref{temp linear dep}), 
\[
w = c v_i =  \sum_{k\neq i} a_k v_k + (c+a_i) v_i + \sum_{j=1}^J b_j u_j,
\]
which implies 
\[
\|\mu+\tau\| \ge \sum_{k\neq i} |a_k| +|c+a_i| + \sum_{j=1}^J |b_j| \ge \|w\|_2.
\]
Combining with $\|\mu+\tau\| \ge \sum_{k=1}^K |c_k|$, we can obtain the bound (\ref{kappa lower bound}). Since $|c|=\|w\|_2$, 
\[
\|\mu\| = \sum_{k\neq i} |c_k| + |c_i+c| + |c| \le \sum_{k=1}^K |c_k| + 2\|w\|_2 \le 3 \kappa(f).
\]

\textbf{Case 4}: $v_{K+1} := w/\|w\|_2 \in S_0\setminus \{\pm v_1,\dots,\pm v_K\}$. By the construction of (\ref{temp form 2}), it holds that $K\le N-2$. We let $\mu = \sum_{k=1}^K c_k \delta_{v_k} +\|w\|_2(\delta_{v_{K+1}}-\delta_{-v_{K+1}})=: \mu_1+\mu_2$, which is supported on $K+2\le N$ points. Let $m \in \{J_1+1,\dots,J\}$ be the index such that $u_m = \pm v_{K+1}$. (If such $m$ does not exists, we set $m=0$ and $b_m=0$ in the following.) Then,
\begin{align*}
\|\mu+\tau\| &= \|\mu_1 + \tau_1\| + \|\mu_2 + b_m (\delta_{u_m}-\delta_{-u_m})\| + \| \tau _2 + \tau_3 - b_m (\delta_{u_m}-\delta_{-u_m})\| \\
&\ge \sum_{k=1}^K \left(|c_k+a_k| + |a_k|\right) + |\|w\|_2 \pm b_m| + \sum_{j\neq m} |b_j|.
\end{align*}
Similar to Case 2, by using equality (\ref{temp linear dep}), we can obtain the bound (\ref{kappa lower bound}). Therefore,
\[
\|\mu\| = \sum_{k=1}^K |c_k| + 2\|w\|_2 \le 3 \kappa(f).
\]

In summary, we have shown that, for the parameterization (\ref{temp form 2}), there exists a measure $\mu$ supported on at most $N$ points such that $f_\mu =f$ and 
\[
\|w\|_2 \lor \sum_{k=1}^K |c_k| \le \kappa(f) \le \|\mu\| \le \sum_{k=1}^K |c_k| + 2\|w\|_2.
\] 
It is easy to construct $\theta \in \bR^{(d+2)N}$ such that $f_\mu =f_\theta$ and $\kappa(\theta)=\|\mu\|$. Hence, $\kappa(f) \le \inf \{ \kappa(\theta): f = f_{\theta}, \theta\in \bR^{(d+2)N} \}\le 3\kappa(f)$.

\subsection{Proof of Theorem \ref{local Gc bound}}\label{sec: proof of Gc}

Without loss of generality, we can assume that $\varsigma=1$ by rescaling. Note that the sub-Gaussian assumption $\bE[\exp(\lambda \xi_i)] \le \exp(\lambda^2/2)$ implies $\bE[\xi_i]=0$ and $\bE[\xi_i^2] \le 1$. It is well known that the Gaussian complexity can be bounded by Dudley's entropy integral \citep[Section 5.3.3]{wainwright2019high}. We will provide a bound for the entropy by using Theorem \ref{app}.

The derivation of the entropy integral bound (\ref{Dudley}) below uses essentially the same argument as in \citet[Lemma A.3]{srebro2010smoothness}, which we repeat for the convenience of the reader. For any $\epsilon>0$, let us denote the $\epsilon$-covering number of the set $\cF$ in $L^2(\mu_n)$ norm by $\cN(\epsilon,\cF,\|\cdot\|_{L^2(\mu_n)})$. In other words, there exists a subset $\{f_1,\dots,f_m\} \subseteq \cF$ of minimal size $m=\cN(\epsilon,\cF,\|\cdot\|_{L^2(\mu_n)})$ such that for any $f\in \cF$, there exits $f_j$ satisfying $\|f-f_j\|_{L^2(\mu_n)} \le \epsilon$, which implies
\begin{align*}
\left| \frac{1}{n} \sum_{i=1}^n \xi_i f(X_i) \right| &\le \left| \frac{1}{n} \sum_{i=1}^n \xi_i f_j(X_i) \right| + \left| \frac{1}{n} \sum_{i=1}^n \xi_i (f(X_i)-f_j(X_i)) \right| \\
&\le \max_{j=1,\dots,m} \left| \frac{1}{n} \sum_{i=1}^n \xi_i f_j(X_i) \right| + \sqrt{\frac{\sum_{i=1}^n \xi_i^2}{n}} \epsilon,
\end{align*}
where we use the Cauchy-Schwarz inequality in the second step. Taking the supremum over $f\in \cF$ and then taking expectations over the noise, we obtain
\begin{equation}\label{Dudley pre}
\bE_{\xi_{1:n}} \left[ \sup_{f\in \cF} \left| \frac{1}{n} \sum_{i=1}^n \xi_i f(X_i) \right|\right] \le \bE_{\xi_{1:n}} \left[ \max_{j=1,\dots,m} \left| \frac{1}{n} \sum_{i=1}^n \xi_i f_j(X_i) \right| \right] + \epsilon,
\end{equation}
where we have used $\bE[(\sum_{i=1}^n \xi_i^2)^{1/2}] \le (\sum_{i=1}^n \bE[\xi_i^2])^{1/2} \le \sqrt{n}$. Define a family of zero-mean random variables indexed by $f\in \cF$ as $Z(f):= \frac{1}{\sqrt{n}} \sum_{i=1}^n \xi_i f(X_i)$. Then, $Z(f)$ is a sub-Gaussian process with respect to the metric $\rho_Z(f,f') := \|f-f'\|_{L^2(\mu_n)}$, meaning that
\[
\bE \left[\exp(\lambda(Z(f)-Z(f')))\right] \le \exp(\lambda^2 \rho_Z(f,f')^2/2), \quad \forall f,f'\in \cF, \lambda\in \bR,
\]
which can be proven by basic properties of sub-Gaussian variables \citep[Exercise 2.13]{wainwright2019high}. Observe that the first term on the right-hand side of (\ref{Dudley pre}) is an expected supremum of the sub-Gaussian process $Z(f)$. We can apply a chaining argument as \citet[Proof of Theorem 5.22]{wainwright2019high} to show that
\[
\bE_{\xi_{1:n}} \left[ \max_{j=1,\dots,m} | Z(f_j) | \right] \le 16 \int_{\epsilon/4}^{D/2} \sqrt{\log \cN(t,\cF,\|\cdot\|_{L^2(\mu_n)})} dt,
\]
where $D:=\sup_{f,f'\in \cF}\rho_Z(f,f')$ denotes the diameter. Therefore, 
\begin{equation}\label{Dudley}
\bE_{\xi_{1:n}} \left[ \sup_{f\in \cF} \left| \frac{1}{n} \sum_{i=1}^n \xi_i f(X_i) \right|\right] \le \inf_{\epsilon\ge 0} \left\{ 4\epsilon + \frac{16}{\sqrt{n}} \int_{\epsilon}^{D/2} \sqrt{\log \cN(t,\cF,\|\cdot\|_{L^2(\mu_n)})} dt \right\}.
\end{equation}

In order to apply this bound to the variation space $\cF_{\sigma}(M)$, we need to estimate its covering number. Since $\|f-f'\|_{L^2(\mu_n)}\le \|f-f'\|_{L^\infty(\bB^d)}$, it is enough to provide a bound for the entropy $\log \cN(\epsilon, \cF_{\sigma}(M), \|\cdot\|_{L^\infty(\bB^d)})$. Note that this bound should be independent of the network width. But we will need to first bound the entropy of $\NN(N,M)$ by $N$ and $M$. 
Recall that any $f\in \NN(N,M)$ can be represented in the form 
\[
f(x) = \sum_{i=1}^{N} a_i\sigma((x^\intercal,1)v_i), \quad v_i \in \bS^d,\ \sum_{i=1}^{N} |a_i| \le M.
\]
It is easy to see that there exists $\epsilon/(2\sqrt{2}M)$-cover $V_\epsilon$ of $\bS^d$ in the metric $\|\cdot\|_2$ with cardinality $|V_\epsilon| \lesssim (1+M/\epsilon)^d$. Similarly, there exists $\epsilon/(2\sqrt{2})$-cover $A_\epsilon$ of the set $A=\{a=(a_1,\dots,a_N): \|a\|_1\le M\}$ in the metric $\|\cdot\|_1$ such that $|A_\epsilon|\lesssim (1+M/\epsilon)^N$. Hence, for any $v_i\in \bS^d$ and $a=(a_1,\dots,a_N) \in A$, we can choose $\widetilde{v}_i\in V_\epsilon$ and $\widetilde{a}=(\widetilde{a}_1,\dots,\widetilde{a}_N) \in A_\epsilon$ such that $\|v_i-\widetilde{v}_i\|_2 \le \epsilon/(2\sqrt{2}M)$ and $\|a-\widetilde{a}\|_1 \le \epsilon/(2\sqrt{2})$. Therefore, letting $\widetilde{f}(x) = \sum_{i=1}^{N} \widetilde{a}_i\sigma((x^\intercal,1)\widetilde{v}_i)$, we have for $x\in \bB^d$,
\begin{align*}
|f(x) - \widetilde{f}(x)| &\le \|a-\widetilde{a}\|_1 \sup_{1\le i \le N} |\sigma((x^\intercal,1)v_i)| + \|\widetilde{a}\|_1 \sup_{1\le i \le N} |\sigma((x^\intercal,1)v_i) - \sigma((x^\intercal,1)\widetilde{v}_i)| \\
&\le \epsilon/2 + \sqrt{2} M \sup_{1\le i \le N} \|v_i - \widetilde{v}_i\|_2 \le \epsilon,
\end{align*}
where we use the Lipschitz continuity of ReLU in the second inequality. Thus,
\begin{equation}\label{covering num bound}
\log \cN(\epsilon, \NN(N,M),\|\cdot\|_{L^\infty(\bB^d)}) 
\lesssim N \log(1+M/\epsilon).
\end{equation}

Now, we estimate the entropy of $\cF_{\sigma}(M)$. By Theorem \ref{app},
\[
\sup_{f \in \cF_{\sigma}(M)}\inf_{f_N \in \NN(N,M)} \|f-f_N\|_{L^\infty(\bB^d)} \lesssim M N^{- \frac{d+3}{2d}}.
\]
Hence, for any $0<\epsilon\le M$, we can choose $N \asymp (\epsilon/M)^{-2d/(d+3)}$ such that for any $f\in \cF_{\sigma}(M)$, there exists $f_N \in \NN(N,M)$ satisfying $\|f-f_N\|_{L^\infty(\bB^d)} <\epsilon/2$. By the triangle inequality, any $\epsilon/2$-cover of $\NN(N,M)$ is an $\epsilon$-cover of $\cF_{\sigma}(M)$. Then, the estimate (\ref{covering num bound}) shows
\begin{align*}
\log \cN(\epsilon, \cF_{\sigma}(M), \|\cdot\|_{L^\infty(\bB^d)}) & \le \log \cN(\epsilon/2, \NN(N,M), \|\cdot\|_{L^\infty(\bB^d)}) \\
& \lesssim N \log(1+2M/\epsilon) \\
& \lesssim (\epsilon/M)^{-2d/(d+3)} \log(1+M/\epsilon).
\end{align*}

Applying (\ref{Dudley}) to $\cF=\{f\in \cF_{\sigma}(M): \|f\|_{L^2(\mu_n)}\le \delta\}$, we have $D=2\delta$ and 
\[
\cG_n(\cF_{\sigma}(M);\delta,\xi_{1:n}) \lesssim \inf_{\epsilon\ge 0} \left\{ 4\epsilon + \frac{16}{\sqrt{n}} \int_\epsilon^\delta (t/M)^{-d/(d+3)} \sqrt{\log(1+M/t)} dt \right\}.
\]
If we choose $\epsilon \asymp \delta^{\frac{3}{d+3}} M^{\frac{d}{d+3}} n^{-1/2}$, then for $n\ge 2$,
\begin{align*}
\cG_n(\cF_{\sigma}(M);\delta,\xi_{1:n}) &\lesssim \frac{\delta^{\frac{3}{d+3}} M^{\frac{d}{d+3}}}{\sqrt{n}} + \frac{M^{\frac{d}{d+3}}}{\sqrt{n}} \sqrt{\log(nM/\delta)} \int_0^\delta t^{-d/(d+3)} dt \\
&\lesssim \frac{\delta^{\frac{3}{d+3}} M^{\frac{d}{d+3}}}{\sqrt{n}} \sqrt{\log(nM/\delta)}.
\end{align*}

Finally, for the star hull $\star(\cT_B\cF_{\sigma}(M))$, we can similarly derive the bound by estimating the covering number. It is easy to see that the covering number of $\cT_B\cF_{\sigma}(M)$ is not larger than that of $\cF_{\sigma}(M)$. For any $\epsilon>0$, let $\{f_1,\dots,f_m\} \subseteq \cT_B\cF_{\sigma}(M)$ be an $\epsilon/2$-cover of $\cT_B\cF_{\sigma}(M)$ and $\{a_1,\dots,a_k\} \subseteq [0,1]$ be an $\epsilon/(2B)$-cover of $[0,1]$. Then, for any $af\in \star(\cT_B\cF_{\sigma}(M))$, we can choose $a_if_j$ such that 
\[
\|af-a_if_j\|_{L^\infty(\bB^d)} \le |a-a_i|\|f\|_{L^\infty(\bB^d)} + |a_i|\|f-f_j\|_{L^\infty(\bB^d)} \le \epsilon.
\]
This shows that $\{a_if_j:1\le i\le k, 1\le j\le m\}$ is $\epsilon$-cover of $\star(\cT_B\cF_{\sigma}(M))$, which implies
\begin{align*}
&\log \cN(\epsilon, \star(\cT_B\cF_{\sigma}(M)), \|\cdot\|_{L^\infty(\bB^d)}) \\
\le & \log \cN(\epsilon, \cF_{\sigma}(M), \|\cdot\|_{L^\infty(\bB^d)}) + \log (1+B/\epsilon) \\
\lesssim & (\epsilon/M)^{-2d/(d+3)} \log(1+M/\epsilon).
\end{align*}
Hence, we can obtain the same bound for the local complexity.

\subsection{Proof of Lemma \ref{oracle inequality}}\label{sec: proof of oracle ineq}

If $f\in \cF_n$ satisfies $\|\widehat{f}_n - f\|_{L^2(\mu_n)} \le \delta_n$, then 
\begin{align*}
\|\widehat{f}_n-h\|_{L^2(\mu_n)}^2 &\le 2 \|f-h\|_{L^2(\mu_n)}^2 + 2\|\widehat{f}_n - f\|_{L^2(\mu_n)}^2 \\
&\le 2 \|f-h\|_{L^2(\mu_n)}^2 + 2\delta_n^2.
\end{align*}
Thus, we can assume $\|\widehat{f}_n - f\|_{L^2(\mu_n)} > \delta_n$.

Observing that $\widehat{f}_n - f \in \partial \cF_n$, we consider the event 
\[
\cA:= \left\{ \exists g\in \partial \cF_n: \|g\|_{L^2(\mu_n)} > \delta_n, \left|\frac{1}{n} \sum_{i=1}^n \eta_i g(X_i) \right| > 2\delta_n \|g\|_{L^2(\mu_n)} \right\}.
\]
If $\cA$ is true, then there exists $\widetilde{g}:= \frac{\delta_n}{\|g\|_{L^2(\mu_n)}} g \in \star(\partial \cF_n)$ such that $\|\widetilde{g}\|_{L^2(\mu_n)} = \delta_n$ and 
\[
\left|\frac{1}{n} \sum_{i=1}^n \eta_i \widetilde{g}(X_i) \right| = \frac{\delta_n}{\|g\|_{L^2(\mu_n)}} \left|\frac{1}{n} \sum_{i=1}^n \eta_i g(X_i) \right| > 2 \delta_n^2.
\]
This implies that $\bP(\cA) \le \bP(Z_n> 2\delta_n^2)$, where $Z_n$ is the random variable defined by
\[
Z_n:= \sup_{\substack{g\in \star(\partial \cF_n) \\ \|g\|_{L^2(\mu_n)}\le \delta_n}} \left| \frac{1}{n} \sum_{i=1}^n \eta_i g(X_i) \right|.
\]
Observe that we can view $Z_n(\eta_{1:n})$ as a convex Lipschitz function of the random vector $\eta_{1:n}$ supported on $[-2B,2B]^n$ with Lipschitz constant $\delta_n/\sqrt{n}$. By the concentration for convex Lipschitz functions on bounded random variables \citep[Theorem 6.10]{boucheron2013concentration}, we conclude that, for any $t>0$,
\begin{equation}\label{remark}
\bP(Z_n> \bE[Z_n] + t) \le \exp\left(-\frac{nt^2}{32 B^2 \delta_n^2}\right).
\end{equation}
Taking $t=\delta_n^2$ and noticing that $\bE[Z_n] = \cG_n(\star(\partial \cF_n);\delta_n,\eta_{1:n}) \le \delta_n^2$ by assumption, we get
\[
\bP(\cA) \le \bP(Z_n> 2\delta_n^2) \le \bP(Z_n> \bE[Z_n] + \delta_n^2) \le \exp\left(-\frac{n \delta_n^2}{32 B^2}\right).
\]
By the definition of the event $\cA$ and the assumption $\|\widehat{f}_n - f\|_{L^2(\mu_n)} > \delta_n$, the inequality 
\[
\left|\frac{1}{n} \sum_{i=1}^n \eta_i \left(\widehat{f}_n(X_i) - f(X_i)\right) \right| \le 2\delta_n \|\widehat{f}_n - f\|_{L^2(\mu_n)}
\]
holds with probability at least $1-\exp(-n\delta_n^2/(32B^2))$. Combining this bound with the base inequality (\ref{base inequality}), we have
\begin{align*}
\|\widehat{f}_n-h\|_{L^2(\mu_n)}^2 &\le \|f-h\|_{L^2(\mu_n)}^2 + 4\delta_n \|\widehat{f}_n - f\|_{L^2(\mu_n)} \\
&\le \|f-h\|_{L^2(\mu_n)}^2 + 16 \delta_n^2 + \frac{1}{4} \|\widehat{f}_n - f\|_{L^2(\mu_n)}^2 \\
&\le \|f-h\|_{L^2(\mu_n)}^2 + 16 \delta_n^2 + \frac{1}{2} \|\widehat{f}_n - h\|_{L^2(\mu_n)}^2 + \frac{1}{2} \|f - h\|_{L^2(\mu_n)}^2 \\
&= \frac{3}{2} \|f-h\|_{L^2(\mu_n)}^2 + 16 \delta_n^2 + \frac{1}{2} \|\widehat{f}_n - h\|_{L^2(\mu_n)}^2.
\end{align*}
Rearranging yields the desired bound.

\subsection{Proof of of Lemma \ref{oracle inequality reg}}\label{sec: proof of oracle ineq reg}

We modify the proofs of Lemma \ref{oracle inequality} and \citet[Theorem 13.17]{wainwright2019high}. For any $f\in \cF_{n,R}$, let us denote $g_f:= \widehat{f}_{n,\lambda_n} - f$ for convenience. We divide the proof into four cases.

\textbf{Case 1}: $\| g_f\|_{L^2(\mu_n)}^2 + \lambda_n \kappa(g_f) \le \delta_n^2$. We have
\begin{align*}
&\|\widehat{f}_{n,\lambda_n}-h\|_{L^2(\mu_n)}^2 + \lambda_n \kappa(\widehat{f}_{n,\lambda_n}) \\
\le & 2 \|f-h\|_{L^2(\mu_n)}^2 + 2 \|g_f\|_{L^2(\mu_n)}^2 + \lambda_n\kappa(g_f) + \lambda_n\kappa(f) \\
\le & 2 \|f-h\|_{L^2(\mu_n)}^2 + \lambda_n\kappa(f) + 2 \delta_n^2.
\end{align*}

\textbf{Case 2}: $\| g_f\|_{L^2(\mu_n)}^2 + \lambda_n \kappa(g_f) > \delta_n^2$ and $\kappa(\widehat{f}_{n,\lambda_n})\le R$. In this case, $g_f = \widehat{f}_{n,\lambda_n} - f \in \partial \cF_{n,R}$. Consider the event 
\[
\cA_1:= \left\{ \exists g\in \partial \cF_{n,R}: \|g\|_{L^2(\mu_n)}^2 + \lambda_n \kappa(g) > \delta_n^2, \left|\frac{1}{n} \sum_{i=1}^n \eta_i g(X_i) \right| > 2\delta_n \sqrt{\|g\|_{L^2(\mu_n)}^2 + \lambda_n \kappa(g)} \right\}.
\]
If $\cA_1$ is true, then there exists 
\[
\widetilde{g}:= \frac{\delta_n}{\sqrt{\|g\|_{L^2(\mu_n)}^2 + \lambda_n \kappa(g)}} g \in \star(\partial \cF_{n,R}),
\]
such that $\|\widetilde{g}\|_{L^2(\mu_n)} \le \delta_n$ and 
\[
\left|\frac{1}{n} \sum_{i=1}^n \eta_i \widetilde{g}(X_i) \right| = \frac{\delta_n}{\sqrt{\|g\|_{L^2(\mu_n)}^2 + \lambda_n \kappa(g)}} \left|\frac{1}{n} \sum_{i=1}^n \eta_i g(X_i) \right| > 2 \delta_n^2.
\]
This implies that $\bP(\cA_1) \le \bP(Z_n(\delta_n)> 2\delta_n^2)$, where $Z_n(\delta_n)$ is the random variable defined by
\[
Z_n(\delta_n):= \sup_{\substack{g\in \star(\partial \cF_{n,R}) \\ \|g\|_{L^2(\mu_n)}\le \delta_n}} \left| \frac{1}{n} \sum_{i=1}^n \eta_i g(X_i) \right|.
\]

As in the proof of Lemma \ref{oracle inequality}, by viewing $Z_n(\delta_n)$ as a convex Lipschitz function of $\eta_{1:n}$, we get, for any $u>0$,
\[
\bP(Z_n(\delta_n)> \bE[Z_n(\delta_n)] + u) \le \exp\left(-\frac{nu^2}{32 B^2 \delta_n^2}\right).
\]
Taking $u=\delta_n^2$ and noticing that $\bE[Z_n(\delta_n)] = \cG_n(\star(\partial \cF_{n,R};\delta_n,\eta_{1:n}) \le \delta_n^2$, we conclude
\[
\bP(\cA_1) \le \bP(Z_n(\delta_n)> 2\delta_n^2) \le \bP(Z_n(\delta_n)> \bE[Z_n(\delta_n)] + \delta_n^2) \le \exp\left(-\frac{n \delta_n^2}{32 B^2}\right).
\]
This implies that the inequality 
\[
\left|\frac{1}{n} \sum_{i=1}^n \eta_i g_f(X_i) \right| \le 2\delta_n \sqrt{\|g_f\|_{L^2(\mu_n)}^2 + \lambda_n \kappa(g_f)}
\]
holds with probability at least $1-\exp(-n\delta_n^2/(32B^2))$. Combining this with inequality (\ref{base inequality reg c_0}), we have
\begin{align*}
& \|\widehat{f}_{n,\lambda_n}-h\|_{L^2(\mu_n)}^2 + \lambda_n \kappa(\widehat{f}_{n,\lambda_n}) \\
\le & c_0 \left(\|f-h\|_{L^2(\mu_n)}^2 + \lambda_n \kappa(f)\right) + 4 \delta_n \sqrt{\|g_f\|_{L^2(\mu_n)}^2 + \lambda_n \kappa(g_f)} \\
\le & c_0 \left(\|f-h\|_{L^2(\mu_n)}^2 + \lambda_n \kappa(f)\right) + 16 \delta_n^2 + \frac{1}{4} \|\widehat{f}_{n,\lambda_n} - f\|_{L^2(\mu_n)}^2 + \frac{1}{4} \lambda_n \kappa(\widehat{f}_{n,\lambda_n} - f) \\
\le & c_0 \left(\|f-h\|_{L^2(\mu_n)}^2 + \lambda_n \kappa(f)\right) + 16 \delta_n^2 + \frac{1}{2} \|\widehat{f}_{n,\lambda_n} - h\|_{L^2(\mu_n)}^2 + \frac{1}{2} \|f - h\|_{L^2(\mu_n)}^2 \\
& \qquad + \frac{1}{4} \lambda_n \kappa(\widehat{f}_{n,\lambda_n}) + \frac{1}{4} \lambda_n \kappa(f) \\
\le & \frac{1+2c_0}{2} \left(\|f-h\|_{L^2(\mu_n)}^2 + \lambda_n \kappa(f) \right)+ 16\delta_n^2 + \frac{1}{2} \left(\|\widehat{f}_{n,\lambda_n} - h\|_{L^2(\mu_n)}^2 + \lambda_n \kappa(\widehat{f}_{n,\lambda_n}) \right).
\end{align*}
Rearranging shows that the desired bound holds with probability at least $1-\exp(-n\delta_n^2/(32B^2))$.

\textbf{Case 3}: $\kappa(\widehat{f}_{n,\lambda_n})> R$ and $\| g_f\|_{L^2(\mu_n)} \le \delta_n \kappa(\widehat{f}_{n,\lambda_n})/R$. Observe that
\begin{align*}
&\widetilde{g}_f:= \frac{R}{\kappa(\widehat{f}_{n,\lambda_n})}g_f = \frac{R}{\kappa(\widehat{f}_{n,\lambda_n})} (\widehat{f}_{n,\lambda_n} - f) \in \partial \cF_{n,R}, \\
&\|\widetilde{g}_f\|_{L^2(\mu_n)} = \frac{R}{\kappa(\widehat{f}_{n,\lambda_n})}\| g_f\|_{L^2(\mu_n)} \le \delta_n.
\end{align*}
We have shown in Case 2 that $\bP(Z_n(\delta_n)> 2\delta_n^2) \le \exp(-n\delta_n^2/(32B^2))$. This implies that the inequality
\[
\left|\frac{1}{n} \sum_{i=1}^n \eta_i g_f(X_i) \right| = \frac{\kappa(\widehat{f}_{n,\lambda_n})}{R} \left|\frac{1}{n} \sum_{i=1}^n \eta_i \widetilde{g}_f(X_i) \right|\le 2\delta_n^2 \frac{\kappa(\widehat{f}_{n,\lambda_n})}{R}
\]
holds with probability at least $1-\exp(-n\delta_n^2/(32B^2))$. By inequality (\ref{base inequality reg c_0}), 
\begin{align*}
\|\widehat{f}_{n,\lambda_n}-h\|_{L^2(\mu_n)}^2 + \lambda_n \kappa(\widehat{f}_{n,\lambda_n}) & \le c_0 \left(\|f-h\|_{L^2(\mu_n)}^2 + \lambda_n \kappa(f)\right) + \frac{4\delta_n^2}{R}\kappa(\widehat{f}_{n,\lambda_n}) \\
& \le c_0 \left(\|f-h\|_{L^2(\mu_n)}^2 + \lambda_n \kappa(f)\right) + \frac{1}{2}\lambda_n \kappa(\widehat{f}_{n,\lambda_n}),
\end{align*}
where we use the assumption $\lambda_n \ge 8\delta_n^2/R$ in the last inequality. Therefore, 
\[
\|\widehat{f}_{n,\lambda_n}-h\|_{L^2(\mu_n)}^2 + \lambda_n \kappa(\widehat{f}_{n,\lambda_n}) \le 2c_0 \left( \|f-h\|_{L^2(\mu_n)}^2 + \lambda_n \kappa(f) \right),
\]
holds with probability at least $1-\exp(-n\delta_n^2/(32B^2))$.

\textbf{Case 4}: $\kappa(\widehat{f}_{n,\lambda_n})> R$ and $\| g_f\|_{L^2(\mu_n)} > \delta_n \kappa(\widehat{f}_{n,\lambda_n})/R$. In this case,
\[
\widetilde{g}_f= \frac{R}{\kappa(\widehat{f}_{n,\lambda_n})}g_f \in \partial \cF_{n,R}, \qquad \|\widetilde{g}_f\|_{L^2(\mu_n)} > \delta_n.
\]
We are going to show that the event
\[
\cA_2:= \left\{ \exists g\in \partial \cF_{n,R}: \|g\|_{L^2(\mu_n)} > \delta_n, \left|\frac{1}{n} \sum_{i=1}^n \eta_i g(X_i) \right| > 2\delta_n \|g\|_{L^2(\mu_n)} + \frac{1}{16} \|g\|_{L^2(\mu_n)}^2 \right\}
\]
holds with small probability. We will prove it by a ``peeling'' argument. To do this, let us denote $t_m = 2^m \delta_n$ for $m=0,1,\dots$, and 
\begin{align*}
\cA_2(t_m):=& \Bigg\{ \exists g\in \partial \cF_{n,R}: t_m<\|g\|_{L^2(\mu_n)}\le 2t_m, \Bigg. \\
&\qquad \qquad\left. \left|\frac{1}{n} \sum_{i=1}^n \eta_i g(X_i) \right| > 2\delta_n \|g\|_{L^2(\mu_n)} + \frac{1}{16} \|g\|_{L^2(\mu_n)}^2 \right\},
\end{align*}
then we have the decomposition $\cA_2 = \cup_{m=0}^\infty \cA_2(t_m)$. To estimate $\bP(\cA_2(t_m))$, we consider the random variables 
\[
Z_n(t):= \sup_{\substack{g\in \star(\partial \cF_{n,R}) \\ \|g\|_{L^2(\mu_n)}\le t}} \left| \frac{1}{n} \sum_{i=1}^n \eta_i g(X_i) \right|, \quad t\ge \delta_n.
\]
Notice that, if the event $\cA_2(t_m)$ holds, then there exits some function $g\in \partial \cF_{n,R}$ with $\|g\|_{L^2(\mu_n)} \in (t_m,t_{m+1}]$ such that 
\begin{align*}
\left|\frac{1}{n} \sum_{i=1}^n \eta_i g(X_i) \right| &> 2\delta_n \|g\|_{L^2(\mu_n)} + \frac{1}{16} \|g\|_{L^2(\mu_n)}^2 \\
& \ge 2t_m \delta_n + \frac{1}{16} t_m^2 = t_{m+1}\delta_n + \frac{1}{64} t_{m+1}^2,
\end{align*}
where we use $t_{m+1}=2t_m$. This lower bound implies 
\[
\bP(\cA_2(t_m)) \le \bP\left(Z_n(t_{m+1}) > t_{m+1}\delta_n + \frac{1}{64} t_{m+1}^2\right).
\]
As in Case 2, by viewing $Z_n(t)$ as a convex Lipschitz function of $\eta_{1:n}$, one obtain 
\[
\bP(Z_n(t)> \bE[Z_n(t)] + u) \le \exp\left(-\frac{nu^2}{32 B^2 t^2}\right),\quad \forall u>0.
\]
Since $\star(\partial \cF_{n,R})$ is star-shaped, the function $\delta \mapsto \cG_n(\star(\partial \cF_{n,R});\delta,\eta_{1:n})/\delta$ is non-increasing. Hence, for any $t\ge \delta_n$,
\[
\bE[Z_n(t)] = \cG_n(\star(\partial \cF_{n,R});t,\eta_{1:n}) \le t \frac{\cG_n(\star(\partial \cF_{n,R});\delta_n,\eta_{1:n})}{\delta_n} \le t\delta_n.
\]
Using this upper bound on the mean and setting $u=t^2/64$, we get
\[
\bP \left(Z_n(t)> t\delta_n + \frac{1}{64}t^2 \right) \le \exp\left(-\frac{nt^2}{2^{17}B^2}\right).
\]
As a consequence,
\begin{align*}
\bP(\cA_2) &\le \sum_{m=0}^\infty \bP(\cA_2(t_m)) \le \sum_{m=1}^\infty \bP\left(Z_n(t_m) > t_m\delta_n + \frac{1}{64} t_m^2\right) \\
& \le \sum_{m=1}^\infty \exp\left(-\frac{nt_m^2}{2^{17}B^2}\right) = \sum_{m=1}^\infty \exp\left(-\frac{2^{2m}n\delta_n^2}{2^{17}B^2}\right) \\
&\le c_1 \exp\left(-\frac{c_2n\delta_n^2}{B^2}\right),
\end{align*}
for some constant $c_1,c_2>0$. 

By the definition of the event $\cA_2$ and the assumption that $\widetilde{g}_f\in \partial \cF_{n,R}$ and $\|\widetilde{g}_f\|_{L^2(\mu_n)} > \delta_n$, we have shown
\[
\left|\frac{1}{n} \sum_{i=1}^n \eta_i \widetilde{g}_f(X_i) \right| \le 2\delta_n \|\widetilde{g}_f\|_{L^2(\mu_n)} + \frac{1}{16} \|\widetilde{g}_f\|_{L^2(\mu_n)}^2
\]
holds with probability at least $1-c_1 \exp(-c_2 n\delta_n^2/B^2)$. Multiplying both sides by $\kappa(\widehat{f}_{n,\lambda_n})/R$, we obtain
\begin{align*}
\left|\frac{1}{n} \sum_{i=1}^n \eta_i g_f(X_i) \right| &\le 2\delta_n \|g_f\|_{L^2(\mu_n)} + \frac{1}{16} \frac{R}{\kappa(\widehat{f}_{n,\lambda_n})} \|g_f\|_{L^2(\mu_n)}^2 \\
&\le 2\delta_n \|g_f\|_{L^2(\mu_n)} + \frac{1}{16} \|g_f\|_{L^2(\mu_n)}^2,
\end{align*}
where we use $\kappa(\widehat{f}_{n,\lambda_n})>R$ in the last inequality. Combining this bound with inequality (\ref{base inequality reg c_0}), we have
\begin{align*}
& \|\widehat{f}_{n,\lambda_n}-h\|_{L^2(\mu_n)}^2 + \lambda_n \kappa(\widehat{f}_{n,\lambda_n}) \\
\le & c_0 \left(\|f-h\|_{L^2(\mu_n)}^2 + \lambda_n \kappa(f)\right) + 4\delta_n \|g_f\|_{L^2(\mu_n)} + \frac{1}{8} \|g_f\|_{L^2(\mu_n)}^2 \\
\le & c_0 \left(\|f-h\|_{L^2(\mu_n)}^2 + \lambda_n \kappa(f)\right) + 32 \delta_n^2 + \frac{1}{4} \|g_f\|_{L^2(\mu_n)}^2 \\
\le & c_0 \left(\|f-h\|_{L^2(\mu_n)}^2 + \lambda_n \kappa(f)\right) + 32 \delta_n^2 + \frac{1}{2} \|\widehat{f}_{n,\lambda_n} - h\|_{L^2(\mu_n)}^2 + \frac{1}{2} \|f - h\|_{L^2(\mu_n)}^2 \\
\le & \frac{1+2c_0}{2} \left(\|f-h\|_{L^2(\mu_n)}^2 + \lambda_n \kappa(f) \right)+ 32 \delta_n^2 + \frac{1}{2} \left(\|\widehat{f}_{n,\lambda_n} - h\|_{L^2(\mu_n)}^2 + \lambda_n \kappa(\widehat{f}_{n,\lambda_n}) \right),
\end{align*}
which yields the desired bound by rearrangements.

\acks{The work described in this paper was partially supported by InnoHK initiative, The Government of the HKSAR, Laboratory for AI-Powered Financial Technologies, the Research Grants Council of Hong Kong [Projects No. CityU 11306220 and 11308020] and National Natural Science Foundation of China [Project No. 12371103] when the second author worked at City University of Hong Kong. We thank the referees for their helpful comments and suggestions on the paper.}

\bibliography{Ref}
\end{document}